%% file: lnn_notes.tex
\newcommand*{\ARXIV}{}

\newcommand*{\CAMREADY}{}

\ifdefined\ARXIV
	\documentclass[10pt]{article}
	\usepackage{libertine} 
	\usepackage{fullpage}
	
	\usepackage{natbib}
	
	\renewcommand{\cite}[1]{\citep{#1}}
	\setcitestyle{authoryear,round,citesep={;},aysep={,},yysep={;}}
	
	\usepackage[utf8]{inputenc} 
	\usepackage[T1]{fontenc}    
	\usepackage{booktabs}       
	\usepackage{multirow}
	\usepackage{microtype}      
	\usepackage{xcolor}         
	\usepackage{footmisc}
	\usepackage{tabularx}
	\usepackage{thmtools}
	\usepackage{thm-restate}
	\usepackage[left=1.25in,right=1.25in,top=1in]{geometry}
	
	\usepackage{comment}
	\usepackage{hyperref}
	\definecolor{mydarkblue}{rgb}{0,0.08,0.5}
	\hypersetup{ %
		pdftitle={},
		pdfauthor={},
		pdfsubject={},
		pdfkeywords={},
		colorlinks=true,
		linkcolor=mydarkblue,
		citecolor=mydarkblue,
		filecolor=mydarkblue,
		urlcolor=mydarkblue
	}
	
	\skip\footins 9pt plus 4pt minus 2pt
	\def\footnoterule{\kern-3pt \hrule width 12pc \kern 2.6pt }
	\leftmargini\leftmargin \leftmarginii 2em
	\renewenvironment{abstract}%
	{%
		\vskip 0in%
		\centerline%
		{\large\bf Abstract}%
		\vspace{-1ex}%
		\begin{quote}%
		}
		{
			\par%
		\end{quote}%
		\vskip 0ex%
		\vspace{-2mm}
	}
	
	\title{\vskip -4ex \bf{Lecture Notes on Linear Neural Networks:} \\[0.5mm] {\normalfont A Tale of Optimization and Generalization in Deep Learning} \vskip -1.5ex}
	
	\author{
		\begin{tabular}{ccc}
 		\textit{Nadav Cohen} & \hspace{3mm} & \textit{Noam Razin}
 		\end{tabular}	
	}
	\date{
		\vspace{-1mm}
		\fontsize{11}{12.2}\selectfont
		Tel Aviv University \\[5mm]
	}
\fi

\ifdefined\NEURIPS
	\PassOptionsToPackage{numbers}{natbib}
	\documentclass{article}
	\ifdefined\CAMREADY
		\usepackage[final]{neurips_2020}
	\else
		\usepackage{neurips_2020}
	\fi
	\usepackage{footmisc}
	\usepackage[utf8]{inputenc} 
	\usepackage[T1]{fontenc}    
	\usepackage{hyperref}       	
	\usepackage{url}            	
	\usepackage{booktabs}       
	\usepackage{amsfonts}       
	\usepackage{nicefrac}       	
	\usepackage{microtype}      
\fi

\ifdefined\CVPR
	\documentclass[10pt,twocolumn,letterpaper]{article}
	\usepackage{footmisc}
	\usepackage{cvpr}
	\usepackage{times}
	\usepackage{epsfig}
	\usepackage{graphicx}
	\usepackage{amsmath}
	\usepackage{amssymb}
	\usepackage[square,comma,numbers]{natbib}
\fi
\ifdefined\AISTATS
	\documentclass[twoside]{article}
	\ifdefined\CAMREADY
		\usepackage[accepted]{aistats2015}
	\else
		\usepackage{aistats2015}
	\fi
	\usepackage{url}
	\usepackage[round,authoryear]{natbib}
\fi
\ifdefined\ICML
	\documentclass{article}
	\usepackage{footmisc}
	\usepackage{microtype}
	\usepackage{graphicx}
	\usepackage{booktabs}
	\usepackage{hyperref}
	
	\ifdefined\CAMREADY
		\usepackage[accepted]{icml2021}
	\else
		\usepackage{icml2021}
	\fi
\fi
\ifdefined\ICLR
	\documentclass{article}
	\usepackage{iclr2019_conference,times}
	\usepackage{footmisc}
	\usepackage{hyperref}
	\usepackage{url}

	\ifdefined\CAMREADY
		\iclrfinalcopy
	\fi
\fi
\ifdefined\COLT
	\ifdefined\CAMREADY
		\documentclass{colt2017}
	\else
		\documentclass[anon,12pt]{colt2017}
	\fi
\	usepackage{times}
\fi

\usepackage{color}
\usepackage{amsfonts}
\usepackage{algorithm}
\usepackage{algorithmic}
\usepackage{graphicx}
\usepackage{nicefrac}
\usepackage{bbm}
\usepackage{wrapfig}
\usepackage{tikz}
\usepackage[titletoc,title]{appendix}
\usepackage{stmaryrd}
\usepackage{comment}
\usepackage{endnotes}
\usepackage{hyperendnotes}
\usepackage{enumerate}
\usepackage{enumitem}
\usepackage[normalem]{ulem}
	
\usepackage{titlesec}
\ifdefined\COLT
\else
	\usepackage{amsmath,amsthm,amssymb,mathtools}
\fi
\usepackage[small]{caption}
\usepackage[caption = false]{subfig}

\usepackage[extdef=true]{delimset}
\usepackage[capitalize,noabbrev]{cleveref}

\ifdefined\COLT
	\newtheorem{claim}[theorem]{Claim}
	\newtheorem{fact}[theorem]{Fact}
	\newtheorem{procedure}{Procedure}
	\newtheorem{conjecture}{Conjecture}	
	\newtheorem{hypothesis}{Hypothesis}	
	
\else
	\newtheorem{lemma}{Lemma}
	\newtheorem{corollary}{Corollary}
	\newtheorem{theorem}{Theorem}
	\newtheorem{proposition}{Proposition}
	\newtheorem{assumption}{Assumption}
	
	\newtheorem{example}{Example}
	\newtheorem{remark}{Remark}
	\newtheorem{claim}{Claim}

	\theoremstyle{definition}
	\newtheorem{definition}{Definition}
	\newtheorem{exercise}{Exercise}
	\newtheorem{supposition}{Supposition}
\fi

\definecolor{green}{rgb}{0.0, 0.5, 0.0}

\definecolor{xcolor-gray}{gray}{0.95}

\include{notation}
\DeclareFontFamily{U}{mathx}{\hyphenchar\font45}
\DeclareFontShape{U}{mathx}{m}{n}{<-> mathx10}{}
\DeclareSymbolFont{mathx}{U}{mathx}{m}{n}
\DeclareMathAccent{\widebar}{0}{mathx}{"73}

\definecolor{darkspringgreen}{rgb}{0.09, 0.45, 0.27}
\usetikzlibrary{decorations.pathreplacing,calligraphy}
\usetikzlibrary{patterns}

\ifdefined\CVPR
	\usepackage[breaklinks=true,bookmarks=false]{hyperref}
	\ifdefined\CAMREADY
		\cvprfinalcopy 
	\fi
	
\fi

\ifdefined\ARXIV
	\newcommand*{\ABBR}{}
\fi
\ifdefined\ICML
	\newcommand*{\ABBR}{}
\fi
\ifdefined\NEURIPS
	\newcommand*{\ABBR}{}
\fi
\ifdefined\ICLR
	\newcommand*{\ABBR}{}
\fi
\ifdefined\COLT
	\newcommand*{\ABBR}{}
\fi
\ifdefined\ABBR
	\newcommand{\eg}{{\it e.g.}}
	\newcommand{\ie}{{\it i.e.}}
	\newcommand{\cf}{{\it cf.}}

\fi

\begin{document}
	
	
	\ifdefined\ARXIV
		\maketitle
	\fi
	\ifdefined\NEURIPS
	\title{Paper Title}
		\author{
			Author 1 \\
			Author 1 Institution \\	
			\texttt{author1@email} \\
			\And
			Author 1 \\
			Author 1 Institution \\	
			\texttt{author1@email} \\
		}
		\maketitle
	\fi
	\ifdefined\CVPR
		\title{Paper Title}
		\author{
			Author 1 \\
			Author 1 Institution \\	
			\texttt{author1@email} \\
			\and
			Author 2 \\
			Author 2 Institution \\
			\texttt{author2@email} \\	
			\and
			Author 3 \\
			Author 3 Institution \\
			\texttt{author3@email} \\
		}
		\maketitle
	\fi
	\ifdefined\AISTATS
		\twocolumn[
		\aistatstitle{Paper Title}
		\ifdefined\CAMREADY
			\aistatsauthor{Author 1 \And Author 2 \And Author 3}
			\aistatsaddress{Author 1 Institution \And Author 2 Institution \And Author 3 Institution}
		\else
			\aistatsauthor{Anonymous Author 1 \And Anonymous Author 2 \And Anonymous Author 3}
			\aistatsaddress{Unknown Institution 1 \And Unknown Institution 2 \And Unknown Institution 3}
		\fi
		]	
	\fi
	\ifdefined\ICML
		\icmltitlerunning{Paper Title}
		\twocolumn[
		\icmltitle{Paper Title} 
		\icmlsetsymbol{equal}{*}
		\begin{icmlauthorlist}
			\icmlauthor{Author 1}{inst} 
			\icmlauthor{Author 2}{inst}
		\end{icmlauthorlist}
		\icmlaffiliation{inst}{Some Institute}
		\icmlcorrespondingauthor{Author 1}{author1@email}
		\icmlkeywords{}
		\vskip 0.3in
		]
		\printAffiliationsAndNotice{} 
	\fi
	\ifdefined\ICLR
		\title{Paper Title}
		\author{
			Author 1 \\
			Author 1 Institution \\
			\texttt{author1@email}
			\And
			Author 2 \\
			Author 2 Institution \\
			\texttt{author2@email}
			\And
			Author 3 \\ 
			Author 3 Institution \\
			\texttt{author3@email}
		}
		\maketitle
	\fi
	\ifdefined\COLT
		\title{Paper Title}
		\coltauthor{
			\Name{Author 1} \Email{author1@email} \\
			\addr Author 1 Institution
			\And
			\Name{Author 2} \Email{author2@email} \\
			\addr Author 2 Institution
			\And
			\Name{Author 3} \Email{author3@email} \\
			\addr Author 3 Institution}
		\maketitle
	\fi

	\input{abstract}

	\ifdefined\COLT
		\medskip
	\fi

	
	\input{body}

	\ifdefined\NEURIPS
		\begin{ack}
			\input{ack}
		\end{ack}
	\else
		\newcommand{\ack}{\input{ack}}
	\fi
	\ifdefined\ARXIV
		\section*{Acknowledgements}
		\input{ack}
	\else
		\ifdefined\COLT
			\acks{\input{ack}}
		\else
			\ifdefined\CAMREADY
				\ifdefined\ICLR
					\newcommand*{\subsuback}{}
				\fi
				\ifdefined\NEURIPS
				\else
					\section*{Acknowledgements}
					\input{ack}
				\fi
			\fi
		\fi
	\fi

	\section*{References}
	{\small
		\ifdefined\ICML
			\bibliographystyle{icml2021}
		\else
			\bibliographystyle{plainnat}
		\fi
		\bibliography{refs}
	}

	\clearpage
	\appendix
	
	
	\ifdefined\ENABLEENDNOTES
		\theendnotes
	\fi
	


\end{document}

%% file: notation.tex
\newcommand{\e}{{\mathbf e}}
\newcommand{\x}{{\mathbf x}}
\newcommand{\y}{{\mathbf y}}
\newcommand{\R}{{\mathbb R}}
\newcommand{\N}{{\mathbb N}}
\newcommand{\D}{{\mathcal D}}

\renewcommand{\P}{{\mathcal P}}
\renewcommand{\S}{{\mathcal S}}
\newcommand{\T}{{\mathcal T}}
\newcommand{\uu}{{\mathbf u}}
\newcommand{\vv}{{\mathbf v}}
\newcommand{\sigmabf}{{\boldsymbol \sigma}}
\def\be   {\begin{equation}}
	\def\ee   {\end{equation}}
\def\bea {\begin{eqnarray}}
	\def\eea {\end{eqnarray}}
\def\beas {\begin{eqnarray*}}
	\def\eeas {\end{eqnarray*}}
\renewcommand{\mapsfrom}{\mathrel{\reflectbox{\ensuremath{\mapsto}}}}
\newcommand\inprodlr[2]{{\left\langle{#1},{#2}\right\rangle}}
\DeclareMathOperator*{\tr}{Tr}
\DeclareMathOperator*{\rank}{Rank}
\DeclareMathOperator*{\argmin}{argmin}  
\newcommand{\inprod}[2]  {\left\langle{#1},{#2}\right\rangle}
\renewcommand{\det}{\mathrm{det}}

%% file: abstract.tex

\begin{abstract}
These notes are based on a lecture delivered by NC on March 2021, as part of an advanced course in Princeton University on the mathematical understanding of deep learning.
They present a theory (developed by NC, NR and collaborators) of linear neural networks---a fundamental model in the study of optimization and generalization in deep learning.
Practical applications born from the presented theory are also discussed.
The theory is based on mathematical tools that are dynamical in nature.
It showcases the potential of such tools to push the envelope of our understanding of optimization and generalization in deep learning.
The text assumes familiarity with the basics of statistical learning theory.\footnote{%
The reader is referred to~\citet{shalev2014understanding} for an introduction to the area.
}
Exercises (without solutions) are included.
\end{abstract}

%% file: body.tex

\section{Introduction} \label{sec:intro}

\emph{Deep learning} (machine learning with neural network models subject to gradient-based training) is delivering groundbreaking performance, which facilitates the rise of artificial intelligence (see~\citet{lecun2015deep}).
However, despite its extreme popularity, our formal understanding of deep learning is limited.
Its application in practice is based primarily on conventional wisdom, trial-and-error and intuition, often leading to suboptimal results (compromising not only effectiveness, but also safety, robustness, privacy, fairness and more). 
Consequently, immense interest in developing mathematical theories behind deep learning has arisen, in the hopes that such theories will shed light on existing empirical findings, and more importantly, lead to principled methods that bring forth improved performance and new capabilities.

From the perspective of statistical learning theory, understanding deep learning requires addressing three fundamental questions: \emph{expressiveness}, \emph{optimization} and \emph{generalization}. 
Expressiveness refers to the ability of compactly sized neural networks to represent functions capable of solving real-world problems. 
Optimization concerns the effectiveness of simple gradient-based algorithms in solving neural network training programs that are non-convex and thus seemingly difficult. 
Generalization treats the phenomenon of neural networks not overfitting even when having much more trainable parameters (weights) than examples to train on.

In these lecture notes we theoretically analyze \emph{linear neural networks}---a fundamental model in the study of optimization and generalization in deep learning.
A linear neural network is a feed-forward fully-connected neural network with linear (no) activation.
Namely, for depth $n \in \N_{\geq 2}$, input dimension $d_0 \in \N$, output dimension $d_n \in \N$, and hidden dimensions $d_1 , d_2 , \ldots , d_{n - 1} \in \N$, it refers to the parametric family of hypotheses $\{ \x \mapsto W_n W_{n - 1} \cdots W_1 \x : W_j \in \R^{d_j \times d_{j-1}} , j = 1 , 2 , \ldots , n \}$, where $W_j$ is regarded as the weight matrix of layer~$j$.
Linear neural networks are trivial from the perspective of expressiveness (they realize only linear input-to-output mappings), but not so in terms of optimization and generalization---they lead to highly non-convex training objectives with multiple minima and saddle points, and when applied to underdetermined learning problems (\eg, matrix sensing; see Subsection~\ref{sec:gen:mat_sense} below) it is unclear a priori what kind of solutions gradient-based algorithms will converge to.
By virtue of these properties, linear neural networks often serve as a theoretical surrogate for practical deep learning, allowing for controlled analyses of optimization and generalization (see, \eg,~\citet{baldi1989neural,fukumizu1998effect,saxe2014exact,kawaguchi2016deep,hardt2016identity,ge2016matrix,gunasekar2017implicit,li2018algorithmic,du2018algorithmic,nar2018step,bartlett2018gradient,laurent2018deep,arora2018optimization,arora2019convergence,arora2019implicit,du2019width,lampinen2019analytic,ji2019gradient,gidel2019implicit,wu2019global,eftekhari2020training,mulayoff2020unique,razin2020implicit,advani2020high,chou2020gradient,li2021towards,yun2021unifying,min2021explicit,tarmoun2021understanding,azulay2021implicit,nguegnang2021convergence,bah2022learning}).

The theory presented in these notes is \emph{dynamical} in nature, relying on careful characterizations of the trajectories assumed in training.
It will demonstrate the potential of dynamical techniques to succeed where other theoretical approaches fail.

Throughout the text, we consider the case where a linear neural network can express any linear (input-to-output) mapping, meaning that its hidden dimensions are large enough to not constrain rank, \ie, $d_j \geq \min \{ d_0 , d_n \}$ for $j = 1 , 2 , \ldots , n - 1$.
Many of the presented results will not rely on this assumption, but for simplicity it is maintained throughout.

\section{Dynamical Analysis} \label{sec:dyn}

Suppose we are given an analytic\footnote{%
A function $f : \D \to \R^d$, where $\D$ is some domain in Euclidean space and $d \in \N$, is said to be \emph{analytic}, if at every $x \in \D$:
$f ( \cdot )$~is infinitely differentiable;
and
the Taylor series of~$f ( \cdot )$ converges to it on some neighborhood of~$x$.
\label{foot:analytic}
}
training loss $\ell : \R^{d_n \times d_0} \to \R$ defined over linear mappings from $\R^{d_0}$ to~$\R^{d_n}$.
For example, $\ell ( \cdot )$~could correspond to a linear regression task with $d_0$ input variables and $d_n$ output variables, or a multinomial logistic regression task with instances in~$\R^{d_0}$ and $d_n$ possible labels.
In line with these examples we typically think of~$\ell ( \cdot )$ as being convex, though unless stated otherwise this will not be assumed.
Applying a linear neural network to~$\ell ( \cdot )$ boils down to optimizing the \emph{overparameterized objective}:
\bea
&\phi : \R^{d_1 \times d_0} \times \R^{d_2 \times d_1} \times \cdots \times \R^{d_n \times d_{n - 1}} \to \R
\text{\,,}&
\label{eq:oprm_obj} \\
&\phi ( W_1 , W_2 , \ldots , W_n ) = \ell ( W_n W_{n - 1} \cdots W_1 )
\text{\,.}&
\nonumber
\eea
While the loss~$\ell ( \cdot )$ may be convex, the following proposition shows that, aside from degenerate cases (namely, aside from cases where~$\ell ( \cdot )$ is globally minimized by the zero mapping), the overparameterized objective~$\phi ( \cdot )$ is non-convex.
\begin{proposition}
\label{prop:oprm_obj_non_convex}
If~$\ell ( \cdot )$ does not attain its global minimum at the origin then~$\phi ( \cdot )$ is non-convex.
\end{proposition}
\begin{proof}
If~$\ell ( \cdot )$ does not attain its global minimum at the origin then the same applies to~$\phi ( \cdot )$.
This means that if $\phi ( \cdot )$ is convex, its gradient at the origin must be non-zero.
However, the latter gradient is clearly zero.
\end{proof}

We are interested in the dynamics of gradient-based optimization when applied to the overparameterized objective~$\phi ( \cdot )$.
Our analysis will focus on \emph{gradient flow}---a continuous version of gradient descent (attained by taking the step size to be infinitesimal).
Under gradient flow, each weight matrix~$W_j$ traverses through a continuous curve, which (with slight overloading of notation) we denote by~$W_j ( \cdot )$.
Accordingly,~$W_j ( 0 )$~represents the value of~$W_j$ at initialization, and~$W_j ( t )$ its value at time~$t \in \R_{> 0}$ of optimization.
With a chosen initialization $W_1 ( 0 ) , W_2 ( 0 ) , \ldots , W_n ( 0 )$, the curves $W_1 ( \cdot ) , W_2 ( \cdot ) , \ldots , W_n ( \cdot )$ are defined through the following set of differential equations:
\be
\dot{W}_j ( t ) := \tfrac{d}{dt} W_j ( t ) = - \tfrac{\partial}{\partial W_j} \, \phi ( W_1 ( t ) , W_2 ( t ) , \ldots , W_n ( t ) )
\text{\,,} ~~ t > 0 ~ , ~ j = 1 , 2 , \ldots , n
\text{\,.}
\label{eq:gf}
\ee
We note that it is possible to translate to gradient descent (with positive step size) many of the gradient flow results which will be derived, either through analogous discrete analyses (see, \eg,~\citet{arora2019convergence}), or by bounding the distance between gradient descent and gradient flow (see~\citet{elkabetz2021continuous}).

\medskip

Moving on to our dynamical analysis, the first observation we make is that optimization conserves certain quantities.
\begin{lemma}
\label{lemma:conserve}
For each $j = 1 , 2 , \ldots , n - 1$, the curve $W_{j + 1}^\top ( \cdot ) W_{j + 1} ( \cdot ) - W_j ( \cdot ) W_j^\top ( \cdot )$ is constant through time, \ie, $W_{j + 1}^\top ( t ) W_{j + 1} ( t ) - W_j ( t ) W_j^\top ( t ) = W_{j + 1}^\top ( 0 ) W_{j + 1} ( 0 ) - W_j ( 0 ) W_j^\top ( 0 )$ for all $t > 0$.
\end{lemma}
\begin{proof}
A straightforward derivation shows that:
\[
    \tfrac{\partial}{\partial W_j} \, \phi ( W_1 ( t ) , W_2 ( t ) , \ldots , W_n ( t ) ) = W_{n : j+1}^\top (t) \nabla \ell(W_{n : 1} (t)) W_{j-1 : 1}^\top (t) \text{\,,} ~~ t > 0 ~,~ j = 1 , 2 , \ldots , n \text{\,,}
\]
where $W_{n : j + 1} (t) := W_{n} (t) W_{n - 1} (t) \cdots W_{j + 1} (t)$ and $W_{j - 1: 1} (t) := W_{j - 1} (t) W_{j - 2} (t) \cdots W_{1} (t)$ for all $j = 1, 2, \ldots, n$ (by convention, $W_{n : n + 1} (t)$ and $W_{0 : 1} (t)$ stand for identity matrices).
Accordingly, under gradient flow (Equation~\eqref{eq:gf}):
\[
    W_{j+1}^\top (t) \dot{W}_{j+1} (t)= \dot{W}_j (t) W_{j}^\top (t) \text{\,,} ~~ t > 0 ~,~ j = 1 , 2 , \ldots , n-1\text{\,.}
\]
For $t > 0$ and $j = 1, 2, \ldots, n - 1$, adding the transpose of the equation above to itself gives:
\[
W_{j+1}^\top (t) \dot{W}_{j+1} (t) + \dot{W}_{j+1}^\top (t) W_{j+1} (t) = W_{j} (t) \dot{W}_{j}^\top (t) + \dot{W}_{j}(t) W_{j}^\top (t)\text{\,.}
\]
Notice that the left-hand side is simply $\tfrac{d}{dt}  (W_{j+1}^\top (t) W_{j+1} (t))$ while the right-hand side is $\tfrac{d}{dt}  (W_{j} (t) W_{j}^\top (t))$.
Hence, integrating with respect to time leads to:
\[
    W_{j+1}^\top (t) W_{j+1} (t) - W_{j+1}^\top (0) W_{j+1} (0) = W_{j} (t) W_{j}^\top (t) - W_{j} (0) W_{j}^\top (0)\text{\,.}
\]
Rearranging the equality completes the proof.
\end{proof}

If the constant matrix to which $W_{j + 1}^\top ( \cdot ) W_{j + 1} ( \cdot ) - W_j ( \cdot ) W_j^\top ( \cdot )$ equals is zero,
then for every $t > 0$ the matrices $W_j ( t )$ and~$W_{j + 1} ( t )$ admit a certain ``balancedness'' (\eg, the non-zero singular values of~$W_j ( t )$ coincide with those of~$W_{j + 1} ( t )$).
We accordingly make the following definition.
\begin{definition}
\label{def:unbalance}
The \emph{unbalancedness magnitude} of the weight matrices $W_1 , W_2 , \ldots , W_n$ is defined to be:
\[
\max\nolimits_{j = 1 , 2 , \ldots , n - 1} \| W_{j + 1}^\top W_{j + 1} - W_j W_j^\top \|_F
\text{\,,}
\]
where $\| \cdot \|_F$ stands for Frobenius norm.
\end{definition}
With Definition~\ref{def:unbalance} in place, Lemma~\ref{lemma:conserve} implies that the unbalancedness magnitude is conserved throughout optimization.
\begin{proposition}
\label{prop:unbalance_conserve}
For every $t > 0$, the unbalancedness magnitude of $W_1 ( t ) , W_2 ( t ) , \ldots , W_n ( t )$ is equal to that of $W_1 ( 0 ) , W_2 ( 0 ) , \ldots , W_n ( 0 )$.
\end{proposition}
\begin{proof}
The result follows immediately from Lemma~\ref{lemma:conserve} and the definition of unbalancedness magnitude (Definition~\ref{def:unbalance}).
\end{proof}
In deep learning, the way that optimization is initialized has vital importance (see~\citet{sutskever2013importance}).
It is common practice to initialize near zero, and we accordingly direct our attention to this regime.
In our context, near-zero initialization implies that unbalancedness magnitude is small when optimization commences.
By Proposition~\ref{prop:unbalance_conserve}, this means that unbalancedness magnitude is small throughout optimization, and in particular, that unbalancedness magnitude becomes smaller and smaller relative to the weight matrices as optimization moves away from the origin.
We will make an idealized assumption of unbalancedness magnitude~zero.
\begin{assumption}
\label{assum:balance_init}
$W_1 ( 0 ) , W_2 ( 0 ) , \ldots , W_n ( 0 )$---the weight matrices at initialization---have unbalancedness magnitude zero.
\end{assumption}
Note that Assumption~\ref{assum:balance_init} does \emph{not} limit the expressiveness of the linear neural network, since any linear mapping can be expressed by a product of weight matrices with unbalancedness magnitude zero (see Procedure~1 in~\citet{arora2019convergence}).
It is possible to extend results derived for unbalancedness magnitude zero to the more general case of small unbalancedness magnitude (see, \eg,~\citet{razin2020implicit}), but for simplicity this will not be demonstrated.

The following definition will serve for reasoning about the dynamics of the input-to-output mapping realized by the network.
\begin{definition}
\label{def:e2e}
The \emph{end-to-end matrix} corresponding to the weight matrices $W_1 , W_2 , \ldots , W_n$ is:
\[
W_{n : 1} := W_n W_{n - 1} \cdots W_1
\text{\,.}
\]
\end{definition}
In accordance with Definition~\ref{def:e2e}, we denote by~$W_{n : 1} ( \cdot )$ the curve traversed by the end-to-end matrix during optimization, \ie, $W_{n : 1} ( \cdot ) := W_n ( \cdot ) W_{n - 1} ( \cdot ) \cdots W_1 ( \cdot )$.
At the heart of our dynamical analysis lies Theorem~\ref{theorem:e2e_dyn} below, which characterizes the dynamics of~$W_{n : 1} ( \cdot )$.
\begin{theorem}[end-to-end dynamics]
\label{theorem:e2e_dyn}
It holds that:
\be
\dot{W}_{n : 1} ( t ) := \tfrac{d}{dt} W_{n : 1} ( t ) = - \sum\nolimits_{j = 1}^n \hspace{-1mm} \left[ W_{n : 1} ( t ) W_{n : 1}^\top ( t ) \right]^\frac{j - 1}{n} \hspace{-1mm} \nabla \ell \big( W_{n : 1} ( t ) \big) \hspace{-1mm} \left[ W_{n : 1}^\top ( t ) W_{n : 1} ( t ) \right]^\frac{n - j}{n}
~~ , ~
t > 0
\text{\,,}
\label{eq:e2e_dyn}
\ee
where~$[\,\cdot\,]^\beta$, $\beta \in \R_{\geq 0}$, stands for a power operator defined over positive semidefinite matrices (with $\beta = 0$ yielding identity by definition).
\end{theorem}
\begin{proof}
For simplicity of presentation, we assume that all weight matrices are square, \ie, that there exists a $d \in \N$ satisfying $d = d_0 = d_1 = \cdots = d_n$.
This assumption can be avoided at the cost of slightly more involved derivations (see proof of Theorem~1 in~\citet{arora2018optimization}).

For $j = 1, 2, \ldots, n$, we denote:
\[
W_{n : j + 1} (t) := W_{n} (t) W_{n - 1} (t) \cdots W_{j + 1} (t)
\text{~~and~~}
W_{j - 1: 1} (t) := W_{j - 1} (t) W_{j - 2} (t) \cdots W_{1} (t)
\text{\,,}
\]
where by convention, $W_{n : n + 1} (t)$ and $W_{0 : 1} (t)$ stand for identity matrices.

Differentiating $W_{n:1} (t)$ with respect to time using the product rule gives:
\be
\begin{split}
	\dot{W}_{n:1} (t) &= \sum \nolimits_{j = 1}^{n} W_{n:j+1} (t) \dot{W}_{j} (t) W_{j-1:1} (t) \\
		&= - \sum\nolimits_{ j = 1}^{n} W_{n: j + 1} (t) W_{n : j+1}^\top (t) \nabla \ell(W_{n : 1} (t)) W_{j-1 : 1}^\top (t) W_{j - 1: 1} (t) \text{\,,}
\end{split}
\label{eq:e2e_time_deriv_proof}
\ee
where the second equality is due to the fact that under gradient flow
\[
\begin{split}
	\dot{W}_j (t) & := - \tfrac{\partial}{\partial W_j} \, \phi ( W_1 ( t ) , W_2 ( t ) , \ldots , W_n ( t ) ) \\
	 &= - W_{n : j+1}^\top (t) \nabla \ell(W_{n : 1} (t) ) W_{j-1 : 1}^\top (t) \text{\,,} ~~ t > 0 ~,~j = 1 , 2 , \ldots , n \text{\,.}
\end{split}
\]
By Assumption~\ref{assum:balance_init} and Proposition~\ref{prop:unbalance_conserve}, the unbalancedness magnitude of $W_1(t), W_2 (t), \ldots, W_n (t)$ is zero for all $t \geq 0$. 
Consequently:
\be
    W_{j + 1}^\top (t) W_{j+1} (t) = W_{j} (t) W_{j}^\top (t) \text{\,,} ~~ t >0 ~,~ j = 1, 2, \ldots, n - 1
     \text{\,.}
     \label{eq:balancedness_eq}
\ee
Fix some $t > 0$.
We construct singular value decompositions for $W_1 (t), W_2 (t), \ldots, W_n (t)$ in an iterative process as follows.\footnote{%
For basic properties of the singular value decomposition see, \eg, Chapter 3 in~\citet{blum2020foundations}.
}
First, let $W_1 (t) = U_1 S_1 V_1^\top$ be an arbitrary singular value decomposition of $W_1 (t)$, \ie, $U_1, V_1 \in \R^{d \times d}$ are orthogonal matrices and $S_1 \in \R^{d \times d}$ is diagonal holding the singular values of~$W_1 (t)$.
Then, for $j = 2, 3, \ldots, n$, notice that any left singular vector $\uu$ of $W_{j - 1} (t)$ with corresponding singular value~$\sigma$ is an eigenvector of $W_{j - 1} (t) W_{j - 1}^\top (t)$ with corresponding eigenvalue~$\sigma^2$, and vice-versa.
Similarly, any right singular vector $\vv$ of $W_j (t)$ with corresponding singular value~$\sigma$ is an eigenvector of $W_j^\top (t) W_j (t)$ with corresponding eigenvalue~$\sigma^2$, and vice-versa.
Thus, Equation~\eqref{eq:balancedness_eq} implies that the left singular vectors of $W_{j - 1} (t)$ are equal to the right singular vectors of~$W_j (t)$, and their corresponding singular values are the same.
As a result, we may take a singular value decomposition $W_j (t) = U_j S_j V_j^\top$ satisfying $S_{j} = S_{j - 1}$ and $V_j = U_{j - 1}$.
Denote $S := S_1 = S_2 = \cdots = S_n$.

With the singular value decompositions described above, for every $j = 1, 2, \ldots, n  - 1$:
\[
\begin{split}
W_{n : j + 1} (t) & = W_n (t) W_{n - 1} (t) \cdots W_{j + 1} (t) \\
& = U_n S_n V_{n}^\top U_{n - 1} S_{n - 1} V_{n - 1}^\top \cdots U_{j + 1} S_{j + 1} V_{j + 1}^\top \\
& = U_n S U_{n - 1}^\top U_{n - 1} S U_{n - 2}^\top \cdots U_{j + 1} S V_{j + 1}^\top \\
& = U_n S^{ n - j } V_{j + 1}^\top
\text{\,.}
\end{split}
\]
Similarly, $W_{j - 1: 1} (t) =  U_{j - 1} S^{j - 1} V_1^\top$ and $W_{n : 1} (t) = U_n S^n V_1^\top$, from which it follows that:
\[
W_{n : j + 1}(t) W_{n : j + 1}^{\top}(t) = [W_{n:1}(t) W_{n:1}^{\top}(t)]^{\frac{n - j}{n}}
\]
and 
\[
W_{j-1:1}^{\top}(t) W_{j-1:1}(t) = [W_{n:1}^{\top}(t) W_{n:1}(t)]^{\frac{j - 1}{n}}
\text{\,.}
\]
Plugging these equalities into Equation~\eqref{eq:e2e_time_deriv_proof} and reversing the summation order concludes the proof.
\end{proof}
To gain insight into the end-to-end dynamics (Theorem~\ref{theorem:e2e_dyn}), we arrange the end-to-end matrix as a vector.
\begin{corollary}[vectorized end-to-end dynamics]
\label{coro:e2e_dyn_vec}
For an arbitrary matrix~$W$, denote by $vec ( W )$ its arrangement as a vector in column-first order.
Then:
\be
vec \big( \dot{W}_{n : 1} ( t ) \big) = - \P \big( W_{n : 1} ( t ) \big) vec \big( \nabla \ell \big( W_{n : 1} ( t ) \big) \big)
~ , ~
t > 0
\text{\,,}
\label{eq:e2e_dyn_vec}
\ee
where $\P : \R^{d_n \times d_0} \to \S_+^{d_n d_0}$, with $\S_+^{d_0 d_n}$ representing the set of $d_0 d_n \times d_0 d_n$ positive semidefinite matrices, is defined as follows.
Given $W \in \R^{d_n \times d_0}$ with singular values $\sigma_1 , \sigma_2 , \ldots , \sigma_{\max \{ d_0 , d_n \}} \geq 0$ (where by definition $\sigma_r = 0$ for $r > \min \{ d_0 , d_n \}$) and corresponding left and right singular vectors $\uu_1 , \uu_2 , \ldots , \uu_{d_n} \in \R^{d_n}$ and $\vv_1 , \vv_2 , \ldots , \vv_{d_0} \in \R^{d_0}$, respectively, $\P ( W ) \in \S_+^{d_0 d_n}$ has eigenvalues:
\[
\sum\nolimits_{j = 1}^n \sigma_r^{2 \frac{n - j}{n}} \sigma_{r'}^{2 \frac{j - 1}{n}}
~~ , ~
r = 1 , 2 , \ldots , d_n
~ , ~
r' = 1 , 2 , \ldots , d_0
\text{\,,}
\]
and corresponding eigenvectors:
\[
vec ( \uu_r \vv_{r'}^\top )
~~ , ~
r = 1 , 2 , \ldots , d_n
~ , ~
r' = 1 , 2 , \ldots , d_0
\text{\,.}
\]
\end{corollary}
\begin{proof}
The proof relies on the notion of \emph{Kronecker product} between matrices.
For any $A \in \R^{d \times m}$ and $B \in \R^{p \times k}$, their Kronecker product is:
\[
A \otimes B := \begin{pmatrix}
	a_{1,1} B & \cdots & a_{1, m} B \\
	\vdots & \ddots & \vdots \\
	a_{d, 1} B & \cdots & a_{d, m} B
\end{pmatrix} \in \R^{dp \times mk}
\text{\,,}
\]
where $a_{i, j}$ stands for the $(i, j)$th entry of $A$, for $i = 1, 2, \ldots, d$ and $j = 1, 2, \ldots, m$. The Kronecker product upholds the following properties, which can be verified directly.
\begin{enumerate}[label={\textbf{P.\arabic*}},leftmargin=3em,topsep=0em]
	\item \label{kron_1} For any matrices $A$ and $B$ such that $AB$ is defined:
	\be
	vec (AB) = ( B^\top \otimes I_{r_A} ) vec(A) = ( I_{c_B} \otimes A ) vec(B)
	\label{eq:kron_1}
	\text{\,,}
	\ee
	where $I_{r_A}$ and $I_{c_B}$ are identity matrices with dimensions corresponding to the number of rows in $A$ and the number of columns in $B$, respectively.
	
	\item \label{kron_2} For any matrices $A_1, A_2, B_1$ and $B_2$ such that $A_1 B_1$ and $A_2 B_2$ are defined:
	\be
	(A_1 \otimes A_2)(B_1 \otimes B_2) = (A_1 B_1) \otimes (A_2 B_2)
	\text{\,.}
	\label{eq:kron_2}
	\ee
	
	\item \label{kron_3} For any matrices $A$ and $B$: $(A \otimes B)^\top = A^\top \otimes B^\top$.
	
	\item \label{kron_4} For any orthogonal matrices $A$ and $B$: $(A \otimes B)^\top = (A \otimes B)^{-1}$.
\end{enumerate}

With the above properties of the Kronecker product in place, we now derive the sought-after expression for $vec (\dot{W}_{n : 1} (t))$.
Fix some time $t > 0$.
By \ref{kron_1}, vectorizing the end-to-end dynamics from Theorem~\ref{theorem:e2e_dyn} leads to:
\begin{align*}
vec \big( \dot{W}_{n : 1} ( t ) \big)  & = - \sum\nolimits_{j = 1}^n  vec \left ( \big[ W_{n : 1} ( t ) W_{n : 1}^\top ( t ) \big ]^\frac{j - 1}{n} \nabla \ell \big( W_{n : 1} ( t ) \big) \big[ W_{n : 1}^\top ( t ) W_{n : 1} ( t ) \big ]^\frac{n - j}{n} \right )
\\ 
&= 
- \sum\nolimits_{j = 1}^n  \left ( \! I_{d_0} \otimes \big[ W_{n : 1} ( t ) W_{n : 1}^\top ( t ) \big ]^\frac{j - 1}{n} \! \right ) \left ( \! \big[ W_{n : 1}^\top ( t ) W_{n : 1} ( t ) \big ]^\frac{n - j}{n} \otimes I_{d_n} \! \right ) \! vec \big (\nabla \ell \big( W_{n : 1} ( t ) \big) \big) 
\text{\,.}
\end{align*}
Then, applying \ref{kron_2}, we arrive at:
\[
vec \big( \dot{W}_{n : 1} ( t ) \big) = -  \sum\nolimits_{j = 1}^n \left (\big[ W_{n : 1}^\top ( t ) W_{n : 1} ( t ) \big ]^\frac{n - j}{n} \otimes \big[ W_{n : 1} ( t ) W_{n : 1}^\top ( t ) \big ]^\frac{j - 1}{n} \right ) vec \big (\nabla \ell \big( W_{n : 1} ( t ) \big ) \big )
\text{\,.}
\]
Denote $Q := \sum\nolimits_{j = 1}^n \big[ W_{n : 1}^\top ( t ) W_{n : 1} ( t ) \big ]^\frac{n - j}{n} \otimes \big[ W_{n : 1} ( t ) W_{n : 1}^\top ( t ) \big ]^\frac{j - 1}{n}$.
It suffices to show that $Q$ adheres to the eigen-characterization of $\P \big( W_{n : 1} ( t ) \big)$.
Let $W_{n : 1} ( t ) = U S V^\top$ be a singular value decomposition of $W_{n : 1} ( t )$, \ie:
$U \in \R^{d_n \times d_n}$ and $V \in \R^{d_0 \times d_0}$ are orthogonal matrices whose respective columns $\uu_1, \uu_2, \ldots, \uu_{d_n} \in \R^{d_n}$ and $\vv_1, \vv_2, \ldots, \vv_{d_0} \in \R^{d_0}$ are left and right singular vectors of $W_{n : 1} (t)$, respectively;
and 
$S \in \R^{d_n \times d_0}$ holds the singular values of $W_{n : 1} (t)$---denoted $\sigma_1, \sigma_2, \ldots, \sigma_{\min \{ d_0, d_n \}} \geq 0$---on its main diagonal, and zeros elsewhere.
Plugging this singular value decomposition into the definition of~$Q$ gives:
\begin{align*}
Q & = \sum\nolimits_{j = 1}^n \big[ V S^\top S V^\top \big ]^\frac{n - j}{n} \otimes \big[ U S S^\top U^\top \big ]^\frac{j - 1}{n} \\
& = \sum\nolimits_{j = 1}^n \Big ( V \big [S^\top S \big ]^\frac{n - j}{n} V^\top \Big ) \otimes \Big ( U \big[ S S^\top \big ]^\frac{j - 1}{n} U^\top \Big ) \\
& = \sum\nolimits_{j = 1}^n \Big (V \otimes U \Big ) \Big ( \big [S^\top S \big ]^\frac{n - j}{n} \otimes \big[ S S^\top \big ]^\frac{j - 1}{n} \Big ) \Big ( V^\top \otimes U^\top \Big ) \\
& = \Big (V \otimes U \Big ) \Big ( \sum\nolimits_{j = 1}^n \big [S^\top S \big ]^\frac{n - j}{n} \otimes \big[ S S^\top \big ]^\frac{j - 1}{n} \Big ) \Big ( V \otimes U \Big )^\top
\text{\,,}
\end{align*}
where the penultimate and last transitions are per \ref{kron_2} and \ref{kron_3}, respectively.
Denoting $O := V \otimes U$ and \smash{$\Lambda := \sum\nolimits_{j = 1}^n [S^\top S ]^\frac{n - j}{n} \otimes [ S S^\top ]^\frac{j - 1}{n}$}, we have that $Q = O \Lambda O^\top$.
Since $V$ and $U$ are orthogonal, per~\ref{kron_4}, so is~$O$.
Furthermore, $\Lambda$ is diagonal, meaning that $O \Lambda O^\top$ is an orthogonal eigenvalue decomposition of~$Q$.
The proof concludes by noticing that the columns of~$O$ are:
\[
vec ( \uu_r \vv_{r'}^\top )
~~ , ~
r = 1 , 2 , \ldots , d_n
~ , ~
r' = 1 , 2 , \ldots , d_0
\text{\,,}
\]
and the corresponding diagonal elements of~$\Lambda$ are:
\[
\sum\nolimits_{j = 1}^n \sigma_r^{2 \frac{n - j}{n}} \sigma_{r'}^{2 \frac{j - 1}{n}}
~~ , ~
r = 1 , 2 , \ldots , d_n
~ , ~
r' = 1 , 2 , \ldots , d_0
\text{\,,}
\]
where by definition $\sigma_r = 0$ for $r > \min \{ d_0 , d_n \}$.
\end{proof}

Comparing the vectorized end-to-end dynamics (Equation~\eqref{eq:e2e_dyn_vec}) to direct minimization of the loss~$\ell ( \cdot )$ via gradient flow (\ie, to $\dot{W} ( t ) {=} - \nabla \ell ( W ( t ) )$, or equivalently, to $vec ( \dot{W} ( t ) ) = - vec ( \nabla \ell ( W ( t ) ) )$), we conclude that the use of a linear neural network boils down to introducing an \emph{implicit preconditioner}~$\P ( W_{n : 1} ( t ) )$, which depends on the location in parameter space~$W_{n : 1} ( t )$.
The eigenvalues and eigendirections of~$\P ( W_{n : 1} ( t ) )$ depend on the singular value decomposition of~$W_{n : 1} ( t )$, such that an increase in the size (singular value) of a singular component of~$W_{n : 1} ( t )$ leads to an increase in eigenvalues of~$\P ( W_{n : 1} ( t ) )$ along eigendirections associated with the singular component.
Qualitatively, this means that the preconditioner~$\P ( W_{n : 1} ( t ) )$ promotes movement along directions that align with the location in parameter space~$W_{n : 1} ( t )$.
With near-zero initialization (regime of interest), $W_{n : 1} ( t )$~can also be regarded as the movement made thus far during optimization.
Accordingly, the preconditioner~$\P ( W_{n : 1} ( t ) )$ may be interpreted as promoting movement in directions already taken, and therefore can be seen as inducing a certain momentum effect.
Implications of this effect to optimization and generalization will be studied in Sections \ref{sec:opt} and~\ref{sec:gen}, respectively.

\begin{remark}
It can be shown (see~\citet{arora2018optimization}) that the end-to-end dynamics induced by a linear neural network (Equation~\eqref{eq:e2e_dyn}) cannot be emulated via \emph{any} modification of the loss~$\ell ( \cdot )$, in particular through regularization.
More precisely, under mild conditions, the end-to-end dynamics cannot be expressed as gradient flow over any objective, in the sense that there exists no continuously differentiable function $f : \R^{d_n \times d_0} \to \R$ satisfying \smash{$\nabla f ( W ) = \sum_{j = 1}^n [ W W^\top ]^\frac{j - 1}{n} \nabla \ell ( W) [ W^\top W ]^\frac{n - j}{n}$} for all $W \in \R^{d_n \times d_0}$.
This may be proven by showing that the vector field $g : \R^{d_n \times d_0} \to \R^{d_n \times d_0}$ defined by \smash{$g ( W ) = \sum_{j = 1}^n [ W W^\top ]^\frac{j - 1}{n} \nabla \ell ( W) [ W^\top W ]^\frac{n - j}{n}$} violates the condition of the fundamental theorem for line integrals (namely, there exist closed curves in~$\R^{d_n \times d_0}$ over which the line integral of~$g ( \cdot )$ does not vanish).
The fact that it is mathematically impossible to mimic the dynamics of linear neural networks via standard means such as regularization, is an indication that a dynamical analysis may lead to new insights beyond those attainable using standard theoretical tools.
Sections \ref{sec:opt} and~\ref{sec:gen} will demonstrate this.
\end{remark}

\section{Optimization} \label{sec:opt}

In this section we study optimization of linear neural networks, \ie, minimization of the overparameterized objective~$\phi ( \cdot )$ defined in Equation~\eqref{eq:oprm_obj}.
Proposition~\ref{prop:oprm_obj_non_convex} in Section~\ref{sec:dyn} has shown that, aside from degenerate cases (ones in which the loss~$\ell ( \cdot )$ is globally minimized by the zero mapping), $\phi ( \cdot )$~is non-convex.
Proposition~\ref{prop:oprm_obj_non_strict} below further establishes that under a mild condition (namely, that $\ell ( \cdot )$ is not locally minimized by the zero mapping, which when $\ell ( \cdot )$ is convex is equivalent to the non-degeneracy condition of~$\ell ( \cdot )$ not being globally minimized by the zero mapping), if the network has three or more layers (\ie, if $n \geq 3$)---setting of interest in the context of deep learning---then $\phi ( \cdot )$ admits non-strict saddle points.\footnote{%
In our context, a non-strict saddle point is a stationary point that is not a local (or global) minimizer, but at which the Hessian is positive semidefinite.
}
Existence of non-strict saddle points poses a hurdle, as it implies that generic landscape arguments from the literature on non-convex optimization (see, \eg,~\citet{ge2015escaping,lee2016gradient}) will not be able to establish minimization of~$\phi ( \cdot )$.
Fortunately, the dynamical analysis of Section~\ref{sec:dyn} will allow us to overcome this hurdle.
\begin{proposition}
\label{prop:oprm_obj_non_strict}
If $\ell ( \cdot )$ does not attain a local minimum at the origin and $n \geq 3$, then $\phi ( \cdot )$ admits non-strict saddle points.
\end{proposition}
\begin{proof}
We will show that $\phi ( \cdot )$ admits a non-strict saddle point at the origin, \ie, at~$( 0 , 0 , \ldots , 0 )$.
The gradient and Hessian of~$\phi ( \cdot )$ at the origin are clearly zero, so it suffices to show that the origin is not a local minimizer of~$\phi ( \cdot )$.
Let $\epsilon > 0$.
By assumption there exists $W \in \R^{d_n \times d_0}$ satisfying $\ell (W) < \ell (0)$ and $\| W \|_F < \epsilon$ (recall that $\| \cdot \|_F$ stands for Frobenius norm).
Without loss of generality suppose that $d_0 \geq d_n$ (the case $d_0 < d_n$ can be treated analogously).
Define $W_1 \in \R^{d_1 \times d_0}$ to be the matrix holding \smash{$\| W \|_F^{- 1 + 1 / n} \cdot W$} in its top $d_n \times d_0$ block, and zeros elsewhere.
Let $I \in \R^{d_n \times d_n}$ be an identity matrix, and for each $j = 2 , 3 , \ldots , n$, define $W_j$ to be the matrix holding \smash{$\| W \|_F^{1 / n} \cdot I$} in its top-left $d_n \times d_n$ block, and zeros elsewhere.
By construction $W_n W_{n - 1} \cdots W_1 = W$, and therefore $\phi (W_1, W_2, \ldots, W_n) = \ell (W) < \ell (0) = \phi (0, 0, \ldots, 0)$.
Since $\| (W_1, W_2, \ldots, W_n) \|_F = \sqrt{n} \cdot \epsilon^{1 / n}$, and $\epsilon$ can be arbitrarily small, the origin (\ie, $(0, 0, \ldots, 0)$) is not a local minimizer of~$\phi(\cdot)$.
\end{proof}

\subsection{Convergence Guarantee} \label{sec:opt:cvg}

In this subsection we employ the dynamical analysis of Section~\ref{sec:dyn} for establishing convergence to global minimum.
The derivation will rely on the concept defined below.
\begin{definition}
\label{def:def_margin}
$W \in \R^{d_n \times d_0}$ is said to have \emph{deficiency margin} $\delta > 0$ if $\ell ( W ) < \ell ( W' )$ for every $W' \in \R^{d_n \times d_0}$ whose minimal singular value is at most~$\delta$.
\end{definition}
The following example illustrates the concept of deficiency margin.
\begin{example}
\label{example:def_margin}
Suppose $\ell ( \cdot )$ is the square loss for a linear regression task with $d_0$ input variables and $d_n$ output variables, \ie, $\ell ( W ) = \tfrac{1}{2 m} \| W X - Y \|_F^2$, where $m \in \N$ is the number of training examples, $X \in \R^{d_0 , m}$~holds training instances as columns, and $Y \in \R^{d_n , m}$~holds training labels as columns (in corresponding order).
Let $\Lambda_{x x} := \tfrac{1}{m} X X^\top \in \R^{d_0 \times d_0}$, $\Lambda_{y y} := \tfrac{1}{m} Y Y^\top \in \R^{d_n \times d_n}$ and $\Lambda_{y x} := \tfrac{1}{m} Y X^\top \in \R^{d_n \times d_0}$ be the empirical (uncentered) instance covariance matrix, label covariance matrix and label-instance cross-covariance~matrix,~respectively.
Using the fact that $\| A \|_F^2 = \tr ( A A^\top )$ for any matrix~$A$, it is straightforward to arrive at $\ell ( W ) = \tfrac{1}{2} \tr ( W \Lambda_{xx} W^\top ) - \tr ( W \Lambda_{y x}^\top ) + \tfrac{1}{2} \tr ( \Lambda_{y y} )$, and if we assume instances are whitened, \ie, $\Lambda_{x x}$ equals identity, then in addition $\ell ( W ) = \tfrac{1}{2} \| W - \Lambda_{y x} \|_F^2 + \text{\emph{const}}$, where $\text{\emph{const}}$ stands for a term that does not depend on~$W$.
This implies that $W  \in \R^{d_n \times d_0}$ has deficiency margin $\delta > 0$ if and only if $\| W - \Lambda_{y x} \|_F < \| W' - \Lambda_{y x} \|_F$ for every $W' \in \R^{d_n \times d_0}$ satisfying $\sigma_{\min} ( W' ) \leq \delta$, where $\sigma_{\min} ( \cdot )$ refers to the minimal singular value of a matrix.
By a standard matrix computation (see Exercise~\ref{exercise:mat_dist}) $\min \{ \| W' - \Lambda_{y x} \|_F : W' \in \R^{d_n \times d_0} , \sigma_{\min} ( W' ) \leq \delta \} = \max \{ 0 , \sigma_{\min} ( \Lambda_{y x} ) - \delta \}$.
We conclude that $W$ has deficiency margin~$\delta$ if and only if $\| W - \Lambda_{y x} \|_F < \sigma_{\min} ( \Lambda_{y x} ) - \delta$, meaning it lies within distance $\sigma_{\min} ( \Lambda_{y x} ) - \delta$ from $\Lambda_{y x}$, the global minimizer of~$\ell ( \cdot )$.
Note that in the case of a single output variable (\ie, $d_n = 1$) the condition $\| W - \Lambda_{y x} \|_F < \sigma_{\min} ( \Lambda_{y x} ) - \delta$ is equivalent to $\| W - \Lambda_{y x} \|_F < \| \Lambda_{y x} \|_F - \delta$, and therefore, if $W$ is randomly sampled from an isotropic distribution concentrated around the origin, the probability of it having a deficiency margin (with some $\delta > 0$) is close to~$0.5$.
\end{example}

Any guarantee of convergence to global minimum (for gradient-based optimization of~$\phi ( \cdot )$) must rely on assumptions that rule out the following scenarios:
\emph{(i)}~the loss~$\ell ( \cdot )$ is pathologically complex (\eg, it entails a myriad of sub-optimal local minimizers and a single global minimizer with a tiny basin of attraction);
and
\emph{(ii)}~optimization is initialized at a sub-optimal stationary point (note that, disregarding the degenerate case where~$\ell ( \cdot )$ is globally minimized by the zero mapping, the origin $( W_1 , W_2 , \ldots , W_n ) = ( 0 , 0 , \ldots , 0 )$ is a sub-optimal stationary point).
We accordingly assume the following:
\emph{(i)}~$\ell ( \cdot )$~is strongly convex\footnote{%
We say that~$\ell ( \cdot )$ is $\alpha$-strongly convex, with $\alpha > 0$, if for any $W , W' \in \R^{d_n \times d_0}$ and $c \in [ 0 , 1 ]$ it holds that $\ell \big( c \cdot W + ( 1 - c ) \cdot W' \big) \leq c \cdot \ell ( W ) + ( 1 - c ) \cdot \ell ( W' ) - \tfrac{1}{2} \alpha c ( 1 - c ) \cdot \| W - W' \|_F^2$.
}
(this does \emph{not} mean that the optimized objective~$\phi ( \cdot )$ is convex---see Propositions \ref{prop:oprm_obj_non_convex} and~\ref{prop:oprm_obj_non_strict});
and
\emph{(ii)}~at initialization, the end-to-end matrix~$W_{n : 1}$ (Definition~\ref{def:e2e}) has a deficiency margin (\ie, it meets the condition in Definition~\ref{def:def_margin} with some $\delta > 0$).
Notice that in the context of Example~\ref{example:def_margin} (square loss for linear regression with whitened instances), $\ell ( \cdot )$~is indeed strongly convex, and as discussed there, in the case of a single output variable, randomly sampling from an isotropic distribution concentrated around the origin leads to a deficiency margin with probability close to~$0.5$.

We are now in a position to present the main result of this subsection---a convergence guarantee born from the dynamical analysis of Section~\ref{sec:dyn}.
For conciseness, we overload notation by using~$\phi ( t )$, with $t \geq 0$, to refer to the value of the optimized objective~$\phi ( \cdot )$ at time~$t$ of optimization, \ie, $\phi ( t ) \,{:=}\, \phi ( W_1 ( t ) , W_2 ( t ) , \ldots , W_n ( t ) )$.
Additionally, we denote by~$\phi^*$ the global minimum of~$\phi ( \cdot )$, \ie, $\phi^* {:=} \inf_{W_j \in \R^{d_j \times d_{j-1}} , j = 1 , 2 , \ldots , n} \phi ( W_1 , W_2 , \ldots , W_n )$.

\begin{theorem}
\label{theorem:lnn_converge}
Assume $\ell ( \cdot )$ is $\alpha$-strongly convex for some $\alpha > 0$, and the end-to-end matrix at initialization $W_{n : 1} ( 0 )$ has deficiency margin~$\delta$ for some $\delta > 0$.
Then, for any $\epsilon > 0$, it holds that $\phi ( t ) - \phi^* \leq \epsilon$ for every time~$t$ satisfying $t \geq \ln \big( (\phi ( 0 ) - \phi^*) / \epsilon \big) \big/ \big( 2 \alpha \delta^{\, 2 ( n - 1 ) / n} \big)$.
\end{theorem}
\begin{proof}
From the chain rule we have:
\[
\tfrac{d}{dt} \ell ( W_{n : 1} ( t ) ) = vec \big ( \nabla \ell ( W_{n : 1} ( t ) ) \big )^\top vec \big ( \dot{W}_{n : 1} (t) \big )
~ , ~
t > 0
\text{\,.}
\]
Plugging in the vectorized end-to-end dynamics (Equation~\eqref{eq:e2e_dyn_vec}) gives:
\[
\tfrac{d}{dt} \ell ( W_{n : 1} ( t ) ) = - vec \big ( \nabla \ell ( W_{n : 1} ( t ) ) \big )^\top \P \big ( W_{n : 1} ( t ) \big ) vec \big ( \nabla \ell ( W_{n : 1} ( t ) ) \big )
~ , ~
t > 0
\text{\,,}
\]
where $\P \big ( W_{n : 1} ( t ) \big )$ is a positive semidefinite matrix upholding \smash{$\lambda_{\min} \big ( \P \big ( W_{n : 1} ( t ) \big ) \big ) \geq \sigma_{\min} \big ( W_{n : 1} (t) \big )^{2 ( n - 1) / n}$}, with $\lambda_{\min} (\cdot)$ and $\sigma_{\min} (\cdot)$ referring to the minimal eigenvalue and the minimal singular value of a matrix, respectively.
Therefore:
\[
\tfrac{d}{dt} \ell ( W_{n : 1} ( t ) ) \leq - \sigma_{\min} \big ( W_{n : 1} (t) \big )^{2 ( n - 1) / n} \cdot \norm{ \nabla \ell ( W_{n : 1} ( t ) ) }_F^2
~ , ~
t > 0
\text{\,.}
\]
Notice that $\tfrac{d}{dt} \ell ( W_{n : 1} ( t ) ) \leq 0$, and so $\ell (W_{n : 1} ( t ) )$ is monotonically non-increasing as a function of~$t$.
Our assumption on $W_{n : 1} (0)$ having deficiency margin~$\delta$ therefore implies that $W_{n : 1} (t)$ has deficiency margin $\delta$ for every $t \geq 0$.
Accordingly, $\sigma_{\min} \big ( W_{n : 1} (t) \big ) > \delta$ for every $t \geq 0$, and we get:
\[
\tfrac{d}{dt} \ell ( W_{n : 1} ( t ) ) \leq - \delta^{2 ( n - 1) / n} \cdot \norm{ \nabla \ell ( W_{n : 1} ( t ) ) }_F^2
~ , ~
t > 0
\text{\,.}
\]
Denote $\ell^* := \inf_{W \in \R^{d_n \times d_0}} \ell (W)$.
Since $\ell (\cdot)$ is $\alpha$-strongly convex, it holds that $\norm{ \nabla \ell ( W ) }_F^2 \geq 2 \alpha ( \ell ( W )  - \ell^* )$ for every $W \in \R^{d_n \times d_0}$.
This leads to:
\[
\tfrac{d}{dt} \ell ( W_{n : 1} ( t ) ) \leq - 2 \alpha \delta^{2 ( n - 1) / n} \big (  \ell ( W_{n : 1} ( t ) )  - \ell^* \big )
~ , ~
t > 0
\text{\,,}
\]
which by Gr\"onwall's inequality implies that:
\[
\ell (W_{n : 1} (t) ) - \ell^* \leq ( \ell( W_{n : 1} (0)) - \ell^*) \exp \big (  - 2 \alpha \delta^{2 ( n - 1) / n} t \big )
~ , ~
t > 0
\text{\,.}
\]
Let $\epsilon > 0$.
Since $\phi^* = \ell^*$ and $\phi (t) = \ell (W_{n : 1} (t))$ for all~$t$, it holds that $\phi (t) - \phi^* \leq \epsilon$ for every~$t$ satisfying $t \geq \ln \big( (\phi ( 0 ) - \phi^*) / \epsilon \big) \big/ \big( 2 \alpha \delta^{\, 2 ( n - 1 ) / n} \big)$, which is what we set out to prove.
\end{proof}

\begin{remark}
The reader is referred to~\citet{arora2019convergence} for a variant of Theorem~\ref{theorem:lnn_converge} that applies to the case where gradient descent has positive (as opposed to infinitesimal) step size.
This variant is derived via discrete arguments analogous to the continuous ones underlying Theorem~\ref{theorem:lnn_converge}.
It establishes the following.
Consider the setting of Example~\ref{example:def_margin} (square loss for linear regression with whitened instances), in which the loss~$\ell ( \cdot )$ has the form $\ell ( W ) = \tfrac{1}{2} \| W - \Lambda_{y x} \|_F^2 + \text{\emph{const}}$, where $\Lambda_{yx} \in \R^{d_n \times d_0}$, and $\text{\emph{const}}$ stands for a term that does not depend on~$W$.
Assume that gradient descent is applied to~$\phi ( \cdot )$ such that at initialization, for some $\delta > 0$, the end-to-end matrix~$W_{n : 1}$ has deficiency margin~$\delta$, and the weight matrices $W_1 , W_2 , \ldots , W_n$ have unbalancedness magnitude (Definition~\ref{def:unbalance}) of at most $( 256 n^3 \| \Lambda_{y x} \|_F^{2 ( n - 1 ) / n} )^{-1} \cdot \delta^2$.\footnote{%
Note that the admission of positive unbalancedness magnitude forms a relaxation of Assumption~\ref{assum:balance_init} on which Theorem~\ref{theorem:lnn_converge} relies.
}
Suppose the step size~$\eta$ of gradient descent equals \smash{$( 6144 n^3 \| \Lambda_{y x} \|_F^{( 6 n - 4 ) / n} )^{-1} \cdot \delta^{( 4 n - 2 ) / n}$} or less.
Then, with $\phi ( t )$ representing the value~of~$\phi ( \cdot )$ after iteration~$t \in \N \cup \{ 0 \}$ of optimization, and $\phi^*$ standing for the global minimum of~$\phi ( \cdot )$, for any $\epsilon > 0$, if $t \geq ( \eta \delta^{2 ( n - 1 ) / n} )^{-1} \cdot \ln ( \phi ( 0 ) / \epsilon )$, it holds that $\phi ( t ) - \phi^* \leq \epsilon$.
\end{remark}

\subsection{Implicit Acceleration by Overparameterization} \label{sec:opt:acc}

\begin{figure}[t]
	\centering
	\includegraphics[width=0.45\textwidth]{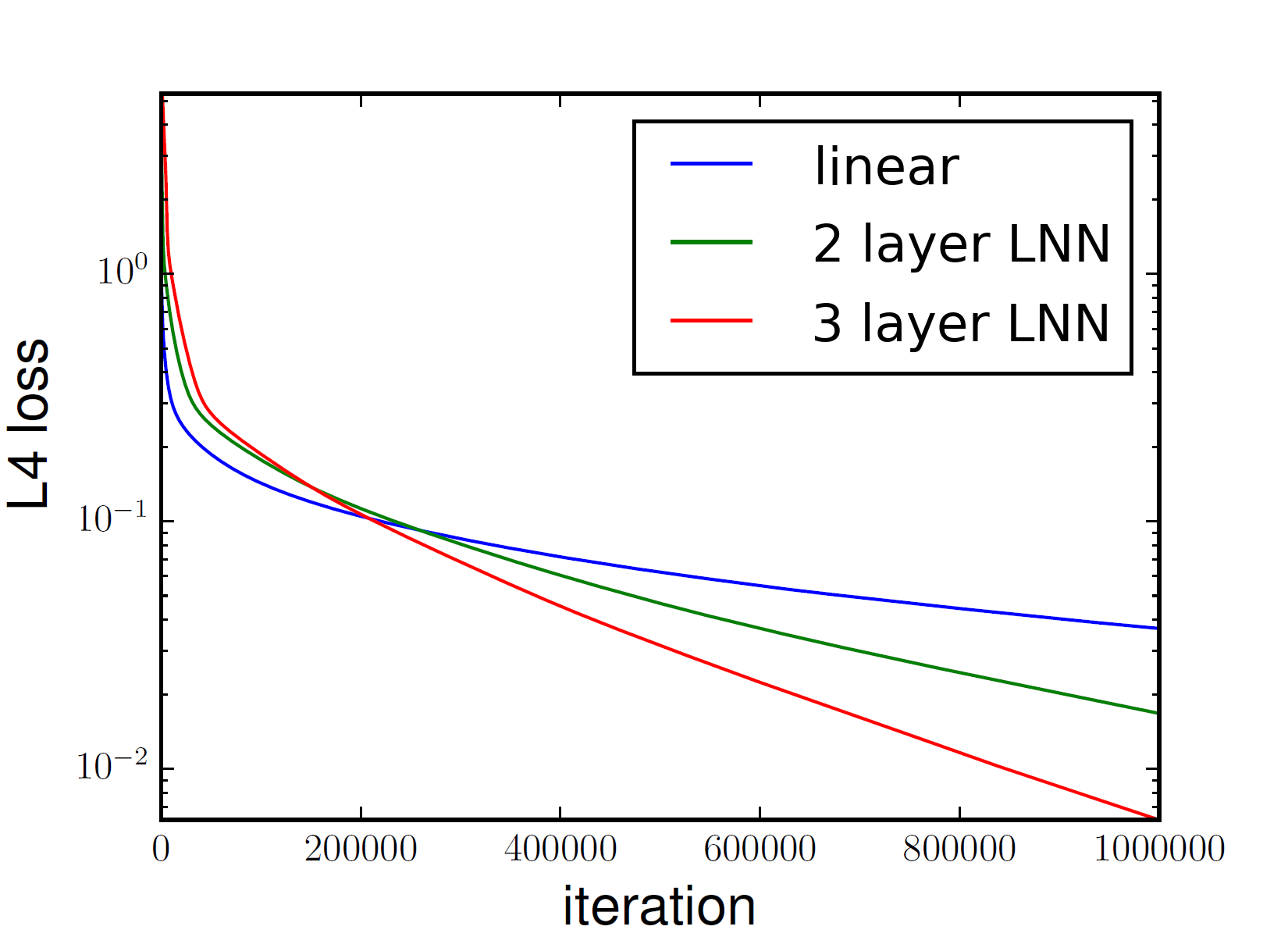} ~
	\includegraphics[width=0.45\textwidth]{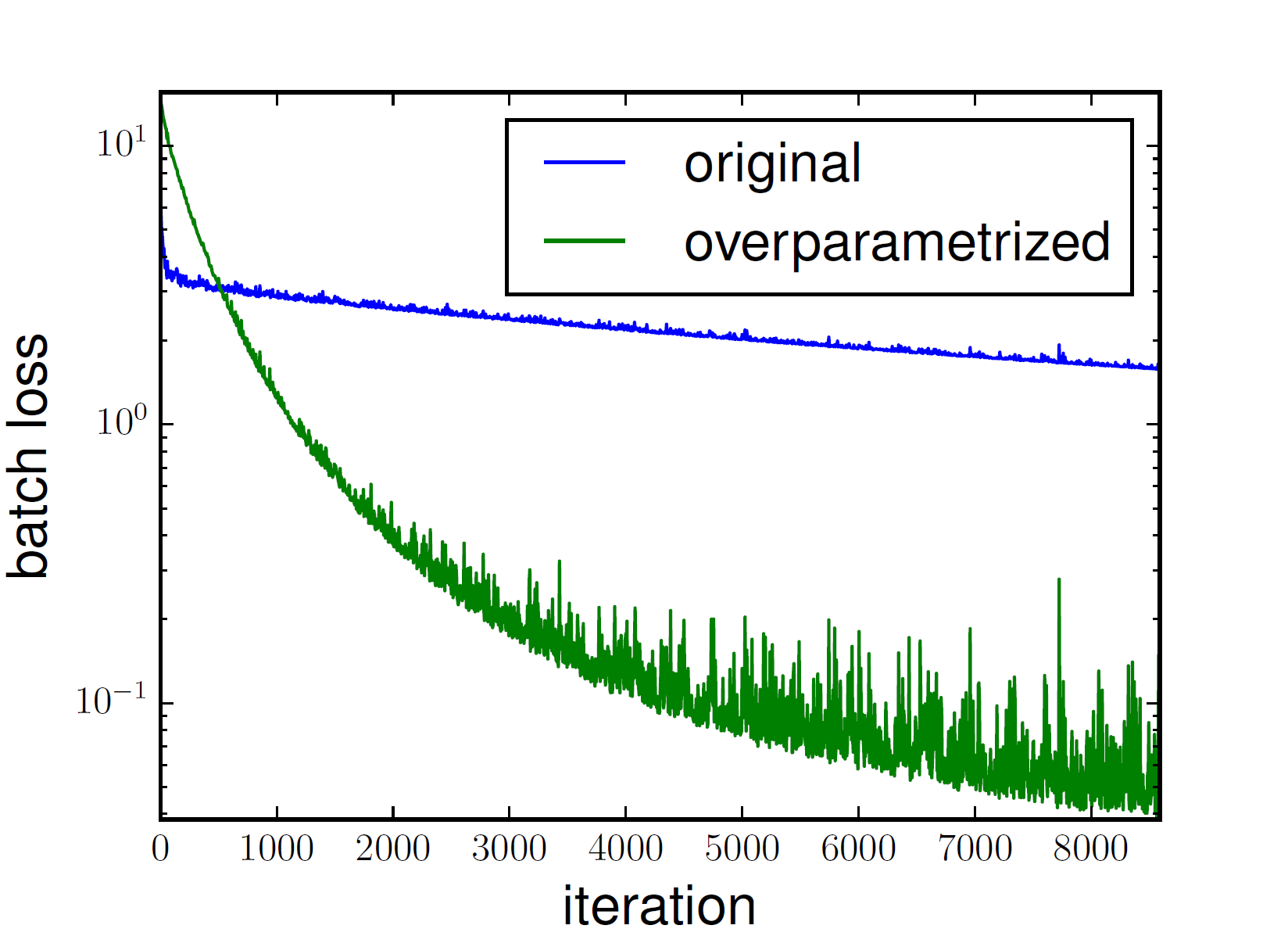}
	\caption{
		Empirical demonstrations of implicit acceleration by overparameterization, \ie, by replacement of linear transformations with linear neural networks.
		Left plot compares gradient descent applied to $\ell_4$~loss for a linear regression task (``linear''), against its application to the overparameterized objectives induced by two and three layer linear neural networks (``2~layer LNN'' and ``3~layer LNN,'' respectively).
		Right plot compares stochastic gradient descent (with momentum) optimizing a non-linear convolutional neural network (``original''), against it optimizing a model obtained from replacing each dense linear transformation with a two layer linear neural network (``overparameterized'').
		For~details see \citet{arora2018optimization}, from which results are taken.
	}
	\label{fig:lnn_acc}
\end{figure}

Subsection~\ref{sec:opt:cvg} employed the dynamical analysis of Section~\ref{sec:dyn} to circumvent the difficulty in establishing convergence to global minimum using landscape arguments (see Proposition~\ref{prop:oprm_obj_non_strict} and preceding text).
In this subsection we discuss a result that further attests to the potential of dynamical analyses.

Classical machine learning is rooted in a preference for optimization to be convex.
In our context, this means that if the loss~$\ell ( \cdot )$ is convex (setting of interest), optimizing it directly is preferred to doing so with a linear neural network, \ie, to minimizing the overparameterized objective~$\phi ( \cdot )$ defined in Equation~\eqref{eq:oprm_obj} (this is because, as Proposition~\ref{prop:oprm_obj_non_convex} in Section~\ref{sec:dyn} shows, aside from degenerate cases, $\phi ( \cdot )$~is non-convex).
A surprising implication of the dynamical analysis of Section~\ref{sec:dyn} is that such preference may be misleading---there exist cases where $\ell ( \cdot )$ is convex, and yet its optimization via gradient descent is much slower than that of~$\phi ( \cdot )$.
We present this result informally below, referring to~\citet{arora2018optimization} for further details.
\begin{claim}[informally stated]
\label{claim:lnn_acc}
Let~$\epsilon > 0$.
Consider minimization of~$\ell ( \cdot )$ via gradient descent with (fixed) step size~$\eta > 0$, and denote by~$t_{\ell , \eta}$ the index of the first iteration in which the value of~$\ell ( \cdot )$ is within~$\epsilon$ from its global minimum (where $t_{\ell , \eta} := \infty$ if this never takes place).
Consider also a discrete minimization of~$\phi ( \cdot )$, or more precisely, a discretization with step size~$\zeta > 0$ of the end-to-end dynamics from Section~\ref{sec:dyn} (with the notations of Corollary~\ref{coro:e2e_dyn_vec}, this amounts to $vec \big( W^{( t + 1 )} \big) \mapsfrom vec \big( W^{( t )} \big) - \zeta \P \big( W^{( t )} \big) vec \big( \nabla \ell ( W^{( t )} ) \big)$ for $t \in \N \cup \{ 0 \}$).
Denote by~$t_{\phi , \zeta}$ the index of the first iteration in this minimization where~$\ell ( \cdot )$ is within~$\epsilon$ from its global minimum (if this never takes place then $t_{\phi , \zeta} := \infty$).
Denote by $t_\ell$ and~$t_\phi$ the optimal convergence times over $\ell ( \cdot )$ and~$\phi ( \cdot )$, respectively, \ie, $t_\ell := \min_{\eta > 0} t_{\ell , \eta}$ and $t_\phi := \min_{\zeta > 0} t_{\phi , \zeta}$.
Let~$c \in \R_{> 0}$ be arbitrarily large, and let $p \in \{ 4 , 6 , 8 , \ldots \}$.
Then, there exist cases where $\ell ( \cdot )$ is the $\ell_p$~loss for a linear regression task (\ie, it has the form $\ell ( W ) = \tfrac{1}{m} \sum\nolimits_{i = 1}^m \| W \x_i - \y_i \|_p^p$ where $\x_1 , \x_2 , \ldots , \x_m \in \R^{d_0}$ are training instances and \smash{$\y_1 , \y_2 , \ldots , \y_m \in \R^{d_n}$} their corresponding labels), thus in particular is convex, and yet a near-zero initialization leads to $t_\ell > c t_\phi$.
\end{claim}
\begin{proof}[Proof sketch]
Let~$W^*$ be a global minimizer of~$\ell ( \cdot )$, and restrict attention to cases where it is unique.
Since $\ell ( \cdot )$ is the $\ell_p$~loss (for a linear regression task) with $p > 2$, its landscape is steep away from~$W^*$, and flat in the vicinity of~$W^*$.
This means (disregarding degenerate cases where $W^*$ is close to the origin) that with near-zero initialization, in order to reach~$W^*$, optimization must initially descend a steep slope, and then traverse through a flat valley.
If optimization is via gradient descent, a small step size is necessary to avoid divergence at the outset, and this leads to slow movement after the landscape flattens.
On the other hand, the (discretized) end-to-end dynamics entail a preconditioner inducing a momentum effect (see text following Corollary~\ref{coro:e2e_dyn_vec}), thus when applied with proper step size they can carefully descend to the flat valley, gradually accelerating thereafter.
Depending on the location of~$W^*$, this can result in arbitrarily faster convergence.
\end{proof}

The takeaway from Claim~\ref{claim:lnn_acc} is that overparameterization with a linear neural network, \ie, insertion of depth via linear layers, can accelerate gradient descent, despite introducing non-convexity while yielding no gain in terms of expressiveness!
This phenomenon of \emph{implicit acceleration by overparameterization} goes beyond the specific cases theoretically demonstrated in Claim~\ref{claim:lnn_acc}.
Namely, it occurs empirically in various cases where $\ell ( \cdot )$ is the $\ell_p$~loss for a linear regression task with $p > 2$---see Figure~\ref{fig:lnn_acc}~(left) for an example.
Moreover, the phenomenon brings forth a practical technique for accelerating optimization of \emph{non-linear} neural networks through addition of linear layers, \ie, through replacement of internal linear transformations with linear neural networks.
This technique, demonstrated in Figure~\ref{fig:lnn_acc}~(right), was employed in various real-world settings (see, \eg,~\citet{bell2019blind,guo2020expandnets,cao2020conv,huh2021low}), and constitutes a practical application of linear neural networks born from a dynamical analysis.

\section{Generalization} \label{sec:gen}

Neural networks are able to generalize even when having much more trainable parameters (weights) than examples to train on.
The fact that this generalization can take place in the absence of any explicit regularization (see \citet{zhang2017understanding} for extensive empirical evidence) has led to a common view by which gradient-based optimization induces an \emph{implicit regularization}---a tendency to fit training examples with functions of low ``complexity.''
It is an ongoing effort to mathematically support this intuition.
The current section does so for linear neural networks.

Optimization of a linear neural network, \ie, minimization of the overparameterized objective~$\phi ( \cdot )$ (defined in Equation~\eqref{eq:oprm_obj}), ultimately produces an end-to-end matrix $W_{n : 1}$ (Definition~\ref{def:e2e}) designed to be a solution for the loss~$\ell ( \cdot )$.
The question we ask in this section is what kind of solution~$W_{n : 1}$ will be produced when $\phi ( \cdot )$ is minimized via gradient descent emanating from near-zero initialization.
This question is most meaningful when $\ell ( \cdot )$ is underdetermined, \ie, admits multiple global minimizers.
A prominent set of tasks giving rise to underdetermined loss functions is \emph{matrix sensing}, which includes linear regression as a special case.
We will focus on settings where $\ell ( \cdot )$ corresponds to a matrix sensing task.
Our treatment will rely on the dynamical analysis of Section~\ref{sec:dyn}.

\subsection{Matrix Sensing} \label{sec:gen:mat_sense}

A \emph{matrix sensing} task is defined by measurement matrices $A_1 , A_2 , \ldots$ $, A_m \in \R^{d_n \times d_0}$ and corresponding measurements \mbox{$b_1 , b_2 , \ldots , b_m \in \R$}.
Given these, the goal is to find a matrix $W \in \R^{d_n \times d_0}$ satisfying $\inprodlr{W}{A_i} := \tr ( W A_i^\top ) = b_i$ for $i = 1 , 2 , \ldots , m$.
A notable special case of matrix sensing, known as \emph{matrix completion}, is where each measurement matrix holds one in a single entry and zeros elsewhere.
Linear regression is also a special case of matrix sensing, as for any $\x \in \R^{d_0}$ and $\y \in \R^{d_n}$, the requirement $W \x = \y$ can be realized via $d_n$ measurement matrices $\e_1 \x^\top , \e_2 \x^\top , \ldots , \e_{d_n} \x^\top$ with corresponding measurements $y_1 , y_2 , \ldots , y_{d_n}$, where, for $i = 1 , 2 , \ldots , d_n$:
$\e_i \in \R^{d_n}$ stands for a vector holding one in its $i$th entry and zeros elsewhere; and 
$y_i \in \R$ represents the $i$th entry of~$\y$.

For tackling matrix sensing, it is common practice to consider the square loss over measurements.
In our context, this amounts to considering a training loss~$\ell ( \cdot )$ of the form:
\be
\ell ( W ) = \tfrac{1}{2 m} \sum\nolimits_{i = 1}^m ( \inprodlr{W}{A_i} - b_i )^2
\text{\,.}
\label{eq:loss_mat_sense}
\ee
If the number of measurements is smaller than the number of entries in the sought-after solution, \ie, if $m < d_0 d_n$, then (assuming the measurement matrices are linearly independent, which generically is the case) the loss~$\ell ( \cdot )$ is underdetermined---it admits infinitely many solutions attaining the global minimum $\ell^* := \inf_{W \in \R^{d_n \times d_0}} \ell ( W ) = 0$.
There is often interest in finding, among all these global minimizers, one whose rank is lowest, \ie, $W^* \in \argmin_{W \in \R^{d_n \times d_0} : \ell ( W ) = \ell^*} \rank ( W )$.
This is NP-hard in general.~%
However, it is known (see~\citet{recht2010guaranteed}) that if $\rank ( W^* )$ is sufficiently low compared to the number of measurements~$m$, and if the measurement matrices satisfy a certain technical condition (``restricted isometry property''), then it is possible to find $W^*$ by solving a convex (constrained) optimization program, namely:
\be
W^* = \argmin\nolimits_{W \in \R^{d_n \times d_0} : \ell ( W ) = \ell^*} \| W \|_*
\text{\,,}
\label{eq:nuc_norm_min}
\ee
where $\| \cdot \|_*$ stands for nuclear norm.\footnote{
The nuclear norm of a matrix is equal to the sum of its singular values.
}
Roughly speaking, this implies that in matrix sensing, given sufficiently many measurements, a global minimizer of lowest~rank can often be found via regularization based on nuclear norm.

\subsection{Implicit Regularization} \label{sec:gen:imp_reg}

Suppose we tackle matrix sensing with a linear neural network, meaning we minimize the overparameterized objective~$\phi ( \cdot )$ (Equation~\eqref{eq:oprm_obj}) induced by a loss~$\ell ( \cdot )$ as defined in Equation~\eqref{eq:loss_mat_sense}.
What kind of solution (for~$\ell ( \cdot )$) will the end-to-end matrix $W_{n : 1}$ reach?
If any of the hidden dimensions of the network (\ie, any of $d_1 , d_2 , \ldots , d_{n - 1}$) were small then $W_{n : 1}$ would be constrained to have low rank, but as stated in Section~\ref{sec:intro}, we consider the case where hidden dimensions are large enough to not restrict the search space (\ie, we assume $d_j \geq \min \{ d_0 , d_n \}$ for $j = 1 , 2 , \ldots , n - 1$).
Surprisingly, experiments show (see, \eg,~\citet{arora2019implicit}) that even in this case, gradient descent with small step size emanating from near-zero initialization tends to produce an end-to-end matrix of low rank.
This tendency is driven by implicit regularization, as there is nothing explicit in the minimized objective~$\phi ( \cdot )$ promoting low rank (indeed,~it typically admits global minimizers whose end-to-end matrices have~high~rank).

Our goal in the current section is to mathematically characterize the implicit regularization described above, \ie, the tendency of linear neural networks to produce a low rank solution (end-to-end matrix)~$W_{n : 1}$ when applied to a loss~$\ell ( \cdot )$ of the form in Equation~\eqref{eq:loss_mat_sense}.
An elegant supposition, formally stated below, is that the implicit regularization solves the convex optimization program in Equation~\eqref{eq:nuc_norm_min}, thus implements a method that under certain conditions provably finds a global minimizer of lowest rank.
\begin{supposition}
\label{supposition:imp_reg_nuc_norm}
If $W_{n : 1}$ converges to a global minimizer (of~$\ell ( \cdot )$), this global minimizer has lowest nuclear norm (among all global minimizers).
\end{supposition}
Supposition~\ref{supposition:imp_reg_nuc_norm} can be proven in scenarios where the measurement matrices $A_1 , A_2 , \ldots , A_m$ satisfy specific conditions (see~\citet{gunasekar2017implicit,li2018algorithmic,arora2019implicit,belabbas2020implicit}).
However, systematic experimentation suggests that it does not hold true in general---see Figure~\ref{fig:lnn_vs_nuc_norm} for an example.
In the next subsection we will employ the dynamical analysis of Section~\ref{sec:dyn} for showing that the implicit regularization of linear neural networks implements a greedy low rank learning process which cannot be characterized as lowering nuclear norm, or \emph{any} other norm.

\begin{figure}[t]
\centering
\includegraphics[width=\textwidth]{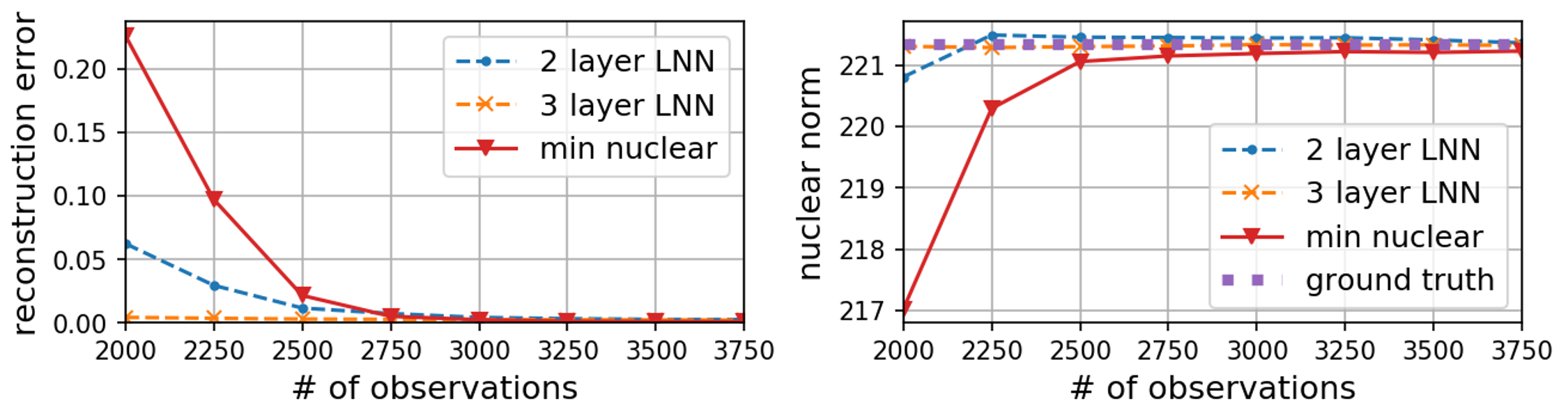}
\caption{
Experiment comparing, on matrix completion (special case of matrix sensing) tasks, the global minimizer of lowest nuclear norm (``min nuclear'') against solutions produced by two and three layer linear neural networks (``2~layer LNN'' and ``3~layer LNN,'' respectively).
Each task entails a different number of observations (measurements), taken from a low rank ground truth matrix.
Left and right plots respectively display reconstruction errors (distances from ground truth) and nuclear norms of the solutions on each task.
Notice that on tasks with many observations, the difference between ground truth and global minimizer of lowest nuclear norm is slight, and the linear neural networks converge to these.
In contrast, on tasks with few observations the difference is significant, and the linear neural networks (especially the deeper one) choose low rank ground truth over global minimizer of lowest nuclear norm.
For details see \citet{arora2019implicit}, from which results are taken.
}
\label{fig:lnn_vs_nuc_norm}
\end{figure}

\subsection{Greedy Low Rank Learning} \label{sec:gen:greedy}

Our analysis of implicit regularization in linear neural networks relies on the concept of analytic singular value decomposition, defined herein for completeness.
\begin{definition}
\label{def:asvd}
For a curve $W : [ 0 , t_e ) \to \R^{d \times d'}$, with $d , d' \in \N$ and $t_e \in \R_{> 0} \cup \{ \infty \}$, an \emph{analytic singular value decomposition} is a triplet $( U : [ 0 , t_e ) \to \R^{d \times \underline{d}} , S : [ 0 , t_e ) \to \R^{\underline{d} \times \underline{d}} , V : [ 0 , t_e ) \to \R^{d' \times \underline{d}} )$, where $\underline{d} := \min \{ d , d' \}$, the curves $U ( \cdot )$, $S ( \cdot )$ and~$V ( \cdot )$ are analytic,\footref{foot:analytic} and for every $t \in [ 0 , t_e )$:
the columns of~$U ( t )$ are orthonormal;
$S ( t )$ is diagonal;\footnote{%
Entries on the diagonal of~$S ( t )$ may be negative, and may appear in any order (in particular, they are generally not arranged in descending or ascending order).
}
the columns of~$V ( t )$ are orthonormal;
and
$W ( t ) = U ( t ) S ( t ) V^\top ( t )$.
\end{definition}
Proposition~\ref{prop:e2e_asvd} below states that the curve traversed by the end-to-end matrix during optimization, \ie, $W_{n : 1} ( \cdot )$, admits an analytic singular value decomposition.
\begin{proposition}
\label{prop:e2e_asvd}
There exists an analytic singular value decomposition for~$W_{n : 1} ( \cdot )$.
\end{proposition}
\begin{proof}
Any analytic curve in matrix space admits an analytic singular value decomposition (see Theorem~1 in~\citet{bunse1991numerical}), so it suffices to show that $W_{n : 1} ( \cdot )$ is~analytic.
Analytic functions are closed under summation, multiplication and composition, so the analyticity of the loss~$\ell ( \cdot )$ implies that the overparameterized objective~$\phi ( \cdot )$ (Equation~\eqref{eq:oprm_obj}) is analytic as well.
A basic result in the theory of analytic differential equations is that gradient flow over an analytic objective yields an analytic curve (see Theorem~1.1 in~\citet{Ilyashenko2008lectures}).
We conclude that $W_1 ( \cdot ) , W_2 ( \cdot ) , \ldots , W_n ( \cdot )$ (curves traversed by the weight matrices during optimization; see Equation~\eqref{eq:gf}) are analytic.
Since $W_{n : 1} ( \cdot ) := W_n ( \cdot ) W_{n - 1} ( \cdot ) \cdots W_1 ( \cdot )$, and, as stated, analytic functions are closed under summation and multiplication, $W_{n : 1} ( \cdot )$~is also analytic.
This concludes the proof.
\end{proof}
Let $( U ( \cdot ) , S ( \cdot ) , V ( \cdot ) )$ be an analytic singular value decomposition for~$W_{n : 1} ( \cdot )$.
For $r = 1 , 2 , \ldots , \min \{ d_0 , d_n \}$, denote by $\uu_r ( \cdot )$ the $r$th column of~$U ( \cdot )$, by $\sigma_r ( \cdot )$ the $r$th diagonal entry of~$S ( \cdot )$, and by $\vv_r ( \cdot )$ the $r$th column of~$V ( \cdot )$.
Up to potential minus signs, $( \sigma_r ( \cdot ) )_{r = 1}^{\min \{ d_0 , d_n \}}$ are the singular values of~$W_{n : 1} ( \cdot )$, with corresponding left and right singular vectors $( \uu_r ( \cdot ) )_{r = 1}^{\min \{ d_0 , d_n \}}$ and $( \vv_r ( \cdot ) )_{r = 1}^{\min \{ d_0 , d_n \}}$, respectively.\footnote{%
More precisely, for every $t \geq 0$, the singular values of~$W_{n : 1} ( t )$ are $( | \sigma_r ( t ) | )_{r = 1}^{\min \{ d_0 , d_n \}}$, and as left and right singular vectors we may take $( \uu_r ( t ) )_{r = 1}^{\min \{ d_0 , d_n \}}$ and $( s_r \vv_r ( t ) )_{r = 1}^{\min \{ d_0 , d_n \}}$ respectively, where $s_r$ equals~$1$ if $\sigma_r ( t ) \geq 0$ and $-1$ otherwise.
}
The following theorem employs the end-to-end dynamics from Section~\ref{sec:dyn} (Theorem~\ref{theorem:e2e_dyn}) for characterizing the dynamics of~$\sigma_r ( \cdot )$.
\begin{theorem}
\label{theorem:sing_dyn}
It holds that:
\[
\dot{\sigma}_r ( t ) := \tfrac{d}{dt} \sigma_r ( t ) = \big( \sigma_r^2 ( t ) \big)^{1 - 1 / n} \inprodlr{- \nabla \ell \big( W_{n : 1} ( t ) \big)}{\uu_r ( t ) \vv_r^\top ( t )} n
~ , ~~ t > 0 ~ , ~ r = 1 , 2 , \ldots , \min \{ d_0 , d_n \}
\text{\,,}
\]
where as before, $\inprodlr{\cdot \,}{\cdot}$~stands for the standard inner product between matrices, \ie, $\inprodlr{A}{B} :=  \tr ( A B^\top )$ for any matrices $A$ and~$B$ of the same dimensions.
\end{theorem}
\begin{proof}
Differentiating the analytic singular value decomposition of $W_{n : 1} ( \cdot )$ with respect to time yields:
\[
\dot{W}_{n : 1}(t) = \dot{U}(t)S(t)V^{\top}(t) +U(t)\dot{S}(t)V^{\top}(t)+U(t)S(t)\dot{V}^{\top}(t) 
~ , ~~ t > 0
\text{\,,}
\]
where $\dot{U} ( t ) := \tfrac{d}{dt} U ( t )$, $\dot{S} ( t ) := \tfrac{d}{dt} S ( t )$ and $\dot{V} ( t ) := \tfrac{d}{dt} V ( t )$.
Multiplying from the left by $U^{\top}(t)$ and from the right by $V(t)$, while using the fact that the columns of~$U ( t )$ are orthonormal and the columns of~$V ( t )$ are orthonormal, we obtain:
\[
U^{\top}(t)\dot{W}_{n : 1}(t)V(t)=U^{\top}(t)\dot{U}(t)S(t)+\dot{S}(t)+S(t)\dot{V}^{\top}(t)V(t)
~ , ~~ t > 0
\text{\,.}
\]
Fix some $r \in \{ 1, 2, \ldots, \min \{ d_0 , d_n \} \}$.
The $r$th diagonal entry of the latter matrix equation is:
\be
\uu_r^{\top}(t)\dot{W}_{n : 1}(t)\vv_r(t)= \uu_r^\top(t) \dot{\uu}_r(t) \cdot \sigma_{r}(t) +\dot{\sigma}_{r}(t)+\sigma_{r}(t) \cdot \dot{\vv}_r^\top (t) \vv_r(t)
~ , ~~ t > 0
\text{\,,}
\label{eq:asvd_deriv_rth_entry}
\ee
where $\dot{\uu}_r(t) := \tfrac{d}{dt} \uu_r(t)$ and $\dot{\vv}_r(t) := \tfrac{d}{dt} \vv_r(t)$.
Since $\uu_r ( \cdot )$ is a column of~$U ( \cdot )$ it has unit length throughout, meaning $\uu_r^\top ( \cdot ) \uu_r ( \cdot ) \equiv 1$.
Similarly, $\vv_r ( \cdot )$ is a column of~$V ( \cdot )$ and therefore $\vv_r^\top ( \cdot ) \vv_r ( \cdot ) \equiv 1$.
This implies that for every $t > 0$:
\[
\uu_r^\top(t) \dot{\uu}_r(t) = \tfrac{1}{2} \tfrac{d}{dt} \big( \uu_r^\top ( t ) \uu_r ( t ) \big) = 0
\text{~~~and~~~}
\dot{\vv}_r^\top(t) \vv_r(t) = \tfrac{1}{2} \tfrac{d}{dt} \big( \vv_r^\top ( t ) \vv_r ( t ) \big) = 0
\text{\,.}
\]
Equation~\eqref{eq:asvd_deriv_rth_entry} thus simplifies to:
\[
\dot{\sigma}_{r}(t) = \uu_r^{\top}(t) \dot{W}_{n : 1}(t) \vv_r(t)
~ , ~~ t > 0
\text{\,.}
\]
Plugging in the end-to-end dynamics (Equation~\eqref{eq:e2e_dyn}), we have:
\be
\dot{\sigma}_{r}(t) = - \sum\nolimits_{j = 1}^n \uu_r^{\top}(t) \left[ W_{n : 1} ( t ) W_{n : 1}^\top ( t ) \right]^\frac{j - 1}{n} \nabla \ell \big( W_{n : 1} ( t ) \big) \left[ W_{n : 1}^\top ( t ) W_{n : 1} ( t ) \right]^\frac{n - j}{n} \vv_r(t)
~ , ~~ t > 0
\text{\,.}
\label{eq:asvd_deriv_rth_entry_e2e_dyn}
\ee
For every $t > 0$ and $j \in \{1, 2, \ldots , n \}$:
\be
\uu_{r}^{\top} (t) \left[ W_{n : 1} ( t ) W_{n : 1}^\top ( t ) \right]^\frac{j - 1}{n} = \uu_{r}^{\top} (t) U ( t ) \left[ S^2 ( t ) \right]^\frac{j - 1}{n} U^\top ( t ) = \big( \sigma_r^2 ( t ) \big)^{ \frac{j - 1}{n} } \cdot \uu_{r}^{\top} (t)
\label{eq:asvd_deriv_rth_entry_e2e_dyn_left_simplify}
\ee
and 
\be
\left[ W_{n : 1}^\top ( t ) W_{n : 1} ( t ) \right]^\frac{n - j}{n} \vv_r (t) = V ( t ) \left[ S^2 ( t ) \right]^\frac{n - j}{n} V^\top ( t ) \vv_r ( t ) = \big( \sigma_r^2 ( t ) \big)^{ \frac{n - j}{n} } \cdot \vv_{r} (t) 
\text{\,.}
\label{eq:asvd_deriv_rth_entry_e2e_dyn_right_simplify}
\ee
Plugging Equations \eqref{eq:asvd_deriv_rth_entry_e2e_dyn_left_simplify} and~\eqref{eq:asvd_deriv_rth_entry_e2e_dyn_right_simplify} into Equation~\eqref{eq:asvd_deriv_rth_entry_e2e_dyn} concludes the proof:
\[
\begin{split}
\dot{\sigma}_{r}(t) &= - \sum\nolimits_{j = 1}^n \big( \sigma_r^2 ( t ) \big)^{1 - 1 / n } \cdot \uu_{r}^{\top} (t) \nabla \ell \big( W_{n : 1} ( t ) \big) \vv_{r} (t) \\
& = \big( \sigma_r^2 ( t ) \big)^{1 - 1 / n} \inprodlr{- \nabla \ell \big( W_{n : 1} ( t ) \big)}{\uu_r ( t ) \vv_r^\top ( t )} n
\end{split}
~~ , ~~ t > 0
\text{\,.}
\]
\end{proof}

An immediate implication of Theorem~\ref{theorem:sing_dyn} is that in the case where $W_{n : 1} ( \cdot )$ is square (\ie, where $d_0 = d_n$), its determinant does not change sign.
\begin{lemma}
\label{lemma:det_sign}
Assume that $d_0 = d_n$.
Then, the sign of~$\det \big( W_{n : 1} ( \cdot ) \big)$ is constant through time.
That is, one of the following holds:
\emph{(i)}~$\det \big( W_{n : 1} ( t ) \big) \,{>}\, 0$ for all $t \,{\geq}\, 0$;
\emph{(ii)}~$\det \big( W_{n : 1} ( t ) \big) \,{=}\, 0$ for all $t \,{\geq}\, 0$;
or
\emph{(iii)}~$\det \big( W_{n : 1} ( t ) \big) \,{<}\, 0$ for all $t \,{\geq}\, 0$.
\end{lemma}
\begin{proof}
Since \smash{$\abs{\det \big( W_{n : 1} ( \cdot ) \big)} \,{=}\, \prod_{r = 1}^{d_0} \abs{ \sigma_r ( \cdot ) }$}, by the continuity of $\det \big( W_{n : 1} ( \cdot ) \big )$ and the intermediate value theorem, it suffices to show that for each $r \in \{1, 2, \ldots, d_0 \}$, either $\sigma_r ( t ) \neq 0$ for all $t \geq 0$ or $\sigma_r ( t ) = 0$ for all $t \geq 0$.

Fix some $r \in \{1, 2, \ldots, d_0 \}$, and let $g ( \cdot )$ be the function defined by:
\[
g ( t ) := \inprod{-\nabla \ell \big( W_{n : 1} ( t ) \big)}{\uu_r ( t ) \vv_r^\top ( t )} n
~ , ~~ t \geq 0
\text{\,.}
\]
By Theorem~\ref{theorem:sing_dyn}, it holds that $\dot{\sigma}_r ( t ) = \big( \sigma_r^2 ( t ) \big)^{1 - 1 / n} g ( t )$ for all $t > 0$.
It can be verified via differentiation that the solution of this differential equation is, when $n = 2$:
\[
\sigma_r ( t )= 
\begin{cases}
\sigma_r ( 0 ) \exp \big( \int\nolimits_{0}^{t} g(t') \,dt' \big) &\text{if}~\sigma_r ( 0 ) > 0 \text{\,,} \\[0.5mm]
\sigma_r ( 0 ) \exp \big( - \int\nolimits_{0}^{t} g(t') \,dt' \big) &\text{if}~\sigma_r ( 0 ) < 0 \text{\,,} \\[0.5mm]
0 &\text{otherwise\,,}	 
\end{cases}
\qquad t \geq 0
\text{\,,}
\]
and when $n \geq 3$:
\[
\sigma_r ( t ) = 
\begin{cases}
\Big( \big( \sigma_r ( 0 ) \big)^{2 / n -1} + (2 / n -1)\int\nolimits_{0}^{t} g(t') dt' \Big)^{\frac{1}{ 2 / n -1}} &\text{if}~\sigma_r ( 0 ) > 0 \text{\,,} \\
- \Big( \big(- \sigma_r ( 0 ) \big)^{2 / n -1} - (2 / n -1)\int\nolimits_{0}^{t} g(t') dt' \Big)^{ \frac{1}{2 / n -1} } &\text{if}~\sigma_r ( 0 ) < 0 \text{\,,} \\[2mm]
0 &\text{otherwise\,,}
\end{cases}
\qquad t \geq 0
\text{\,.}
\]
Whether $n = 2$ or $n \geq 3$, it holds that either $\sigma_r ( t ) \neq 0$ for all $t \geq 0$ or $\sigma_r ( t ) = 0$ for all $t \geq 0$.
This concludes the proof.
\end{proof}
Lemma~\ref{lemma:det_sign} has a far-reaching consequence: there exist cases where $\ell ( \cdot )$ corresponds to a matrix sensing task (\ie, is of the form in Equation~\eqref{eq:loss_mat_sense}), and its minimization with a linear neural network leads \emph{all} norms of~$W_{n : 1} ( \cdot )$ to \emph{grow towards infinity}.
This result---formally delivered by Proposition~\ref{prop:imp_reg_no_norm} below---contrasts Supposition~\ref{supposition:imp_reg_nuc_norm} in that it means the implicit regularization of linear neural networks in matrix sensing cannot be characterized as lowering \emph{any} norm, in particular the nuclear norm.
\begin{proposition}
\label{prop:imp_reg_no_norm}
Assume that $d_0 = d_n \geq 2$.
Then, there exist measurement matrices $A_1 , A_2 , \ldots , A_m \in \R^{d_n \times d_0}$ and corresponding measurements $b_1 , b_2 , \ldots , b_m \in \R$ such that the loss $\ell ( \cdot )$ defined by Equation~\eqref{eq:loss_mat_sense} admits solutions attaining the~global~minimum $\ell^* := \inf_{W \in \R^{d_n \times d_0}} \ell ( W ) = 0$, and yet the following holds.
For~any~norm $\| \cdot \|$ over~$\R^{d_n \times d_0}$, there exist constants $c \in \R_{> 0}$ and $c' \in \R$ such that if $\det \big( W_{n : 1} ( 0 ) \big) > 0$, \ie, if the determinant of the end-to-end matrix is positive at initialization,\footnote{%
The condition $\det \big( W_{n : 1} ( 0 ) \big) > 0$ can be replaced by $\det \big( W_{n : 1} ( 0 ) \big) < 0$ (see Exercise~\ref{exercise:imp_reg_no_norm_neg_det}).
}
then:
\[
\| W_{n : 1} ( t ) \| \geq c \big( \ell \big( W_{n : 1} ( t ) \big) - \ell^* \big)^{- 1 / 2} + c'
~ , ~
t \geq 0
\text{\,,}
\]
meaning in particular that $\| W_{n : 1} ( \cdot ) \|$ diverges to infinity when the loss approaches global minimum (\ie, when $\ell \big( W_{n : 1} ( \cdot ) \big)$ converges to~$\ell^*$).
\end{proposition}
\begin{proof}
For simplicity of presentation, assume that $d_0 = 2$.
The proof below can easily be extended to account for arbitrary $d_0 \geq 2$ (see Appendix~B in~\citet{razin2020implicit}).

Consider the measurement matrices $A_1 =  \e_1 \e_2^\top$, $A_2 = \e_2 \e_1^\top$ and $A_3 = \e_2 \e_2^\top$, where $\e_1 := ( 1 , 0 )^\top \in \R^2$ and $\e_2 := ( 0 , 1 )^\top \in \R^2$, with corresponding measurements $b_1 = 1$, $b_2 = 1$ and $b_3 = 0$.
The loss~$\ell ( \cdot )$ in this case can be written as:
\be
\ell (W) = \frac{1}{6} \left [ (w_{1, 2} - 1)^2 + (w_{2, 1} - 1)^2 + w_{2, 2}^2 \right ]
\label{eq:loss_imp_reg_no_norm}
\text{\,,}
\ee
where $w_{i, j}$ denotes the $(i, j)$th entry of $W \in \R^{2 \times 2}$, for $i = 1,2$ and $j = 1, 2$.
Clearly, any $W \in \R^{2 \times 2}$ holding ones in its off-diagonal and zero in its bottom-right entry attains the global minimum $\ell^* = 0$.

Fix some time $t \geq 0$, and for the moment, assume that $\ell ( W_{n : 1} (t) ) < 1 / 6$ (this assumption will later be lifted).
We slightly overload notation by using $w_{i, j} (t)$ to denote the $(i, j)$th entry of~$W_{n : 1} (t)$, for $i = 1, 2$ and $j = 1, 2$.
By Equation~\eqref{eq:loss_imp_reg_no_norm}:
\be
\abs{ w_{1,2} (t) - 1 } \leq \sqrt{6 \ell ( W_{n : 1} (t) ) } ~~~,~~~ \abs{ w_{2,1} (t) - 1 } \leq \sqrt{6 \ell ( W_{n : 1} (t) ) } ~~~,~~~ \abs{ w_{2, 2} (t) } \leq \sqrt{ 6 \ell ( W_{n : 1} (t) ) } 
\text{\,.}
\label{eq:seen_entries_bounds}
\ee
Lemma~\ref{lemma:det_sign} states that the determinant of $W_{n : 1} ( \cdot )$ does not change sign during optimization.
This implies that if $\det  ( W_{n : 1} ( 0 ) ) > 0$ then:
\be
\det \big(W_{n : 1} ( t ) \big ) = w_{1,1} (t) \cdot w_{2,2} (t) -  w_{1, 2} (t) \cdot w_{2,1} (t) > 0
\text{\,.}
\label{eq:det_greater_than_zero}
\ee
Equation~\eqref{eq:seen_entries_bounds}, along with our (temporary) assumption $\ell ( W_{n : 1} (t) ) < 1 / 6$, ensures that $w_{1, 2} (t) > 0$ and $w_{2, 1} (t) > 0$.
Thus, for Equation~\eqref{eq:det_greater_than_zero} to hold, necessarily $w_{1,1} (t) \cdot w_{2, 2} (t) > 0$.
We may therefore write $\abs{ w_{1,1} (t) \cdot w_{2,2} (t) } = w_{1,1} (t) \cdot w_{2,2} (t) > w_{1,2} (t) \cdot w_{2,1} (t)$.
Dividing by~$\abs{ w_{2, 2} (t) }$, while applying the bounds from Equation~\eqref{eq:seen_entries_bounds}, leads to:
\[
\abs{ w_{1, 1} (t) } > \frac{ \brk1{ 1 - \sqrt{ 6 \ell( W_{n : 1} (t) ) } }^2 }{ \sqrt{ 6 \ell( W_{n : 1} (t) ) } } = \frac{ 1 }{ \sqrt{ 6 \ell( W_{n : 1} (t) ) } } - 2 + \sqrt{ 6 \ell( W_{n : 1} (t) ) } \geq \frac{ 1 }{ \sqrt{ 6 \ell( W_{n : 1} (t) ) } } - 2
\text{\,.}
\]
It remains to convert this lower bound on $\abs{ w_{1, 1} (t) }$ to a lower bound on $\norm{ W_{n : 1} (t) }$.
Since $\| \cdot \|$ upholds the triangle inequality and absolute homogeneity:
\[
\norm{ W_{n : 1} (t) } \geq \abs{ w_{1,1} (t) } \cdot \norm{ \e_1 \e_1^\top } - \norm{ W_{n : 1} (t) - w_{1,1} (t) \cdot \e_1 \e_1^\top} 
\text{\,.}
\]
Using again the triangle inequality and absolute homogeneity, with subsequent application of Equation~\eqref{eq:seen_entries_bounds}, we bound the latter term on the right-hand side above:
\[
\begin{split}
\norm{ W_{n : 1} (t) - w_{1,1} (t) \cdot \e_1 \e_1^\top} & \leq \big ( 1 + \sqrt{ 6 \ell (W_{n : 1} (t) ) } \big ) \cdot \big ( \norm{ \e_1 \e_2^\top } + \norm{ \e_2 \e_1^\top } \big ) + \sqrt{ 6 \ell ( W_{n : 1} (t) ) } \cdot \norm{ \e_2 \e_2^\top } \\
& \leq 2 \big ( \norm{ \e_1 \e_2^\top } + \norm{ \e_2 \e_1^\top } + \norm{ \e_2 \e_2^\top } \big )
\text{\,,}
\end{split}
\]
where the second transition is by our (temporary) assumption $\ell ( W_{n : 1} (t)  ) < 1 / 6$.
Combining the latter three inequalities, we obtain:
\[
\norm{ W_{n : 1} (t) } \geq  \frac{ 1 }{ \sqrt{ 6 \ell( W_{n : 1} (t) ) } } \cdot \norm{ \e_1 \e_1^\top } - 2 \big ( \norm{ \e_1 \e_1^\top } +  \norm{ \e_1 \e_2^\top } + \norm{ \e_2 \e_1^\top } + \norm{ \e_2 \e_2^\top } \big )
\text{\,.}
\]
Defining the constants $c := \norm{ \e_1 \e_1^\top } / \sqrt{6} > 0$ and $c' := \min \big \{ - \sqrt{6} c \, , - 2 \big ( \norm{ \e_1 \e_1^\top } +  \norm{ \e_1 \e_2^\top } + \norm{ \e_2 \e_1^\top } + \norm{ \e_2 \e_2^\top } \big ) \big \}$, while recalling that $\ell^* = 0$, we have:
\be
\| W_{n : 1} ( t ) \| \geq c \big( \ell \big( W_{n : 1} ( t ) \big) - \ell^* \big)^{- 1 / 2} + c'
\text{\,.}
\label{eq:norm_bound}
\ee
This is the sought-after result, subject to our (temporary) assumption $\ell ( W_{n : 1} (t) ) < 1 / 6$.
To lift the assumption, note that whenever $\ell ( W_{n : 1} (t) ) \geq 1 / 6$, the right-hand side of Equation~\eqref{eq:norm_bound} is non-positive, hence it is trivially no greater than the left-hand side.
\end{proof}

The inability of norms to explain the implicit regularization under study (\ie, that of linear neural networks in matrix sensing) poses a quandary, as from a classical machine learning perspective, regularization is typically norm-based.
To overcome this quandary we return to Theorem~\ref{theorem:sing_dyn}, which provides an alternative explanation by revealing a form of \emph{greedy low rank learning}.
Namely, Theorem~\ref{theorem:sing_dyn} reveals that each~$\sigma_r ( \cdot )$, \ie, each singular value of~$W_{n : 1} ( \cdot )$, evolves at a rate given by a product of two factors:
\emph{(i)}~$\inprodlr{- \nabla \ell \big( W_{n : 1} ( \cdot ) \big)}{\uu_r ( \cdot ) \vv_r^\top ( \cdot )} n$, which implies that $\sigma_r ( \cdot )$ moves faster when its singular component (\ie, $\uu_r ( \cdot ) \vv_r^\top ( \cdot )$) is more aligned with the direction of steepest descent in the loss (\ie, with $- \nabla \ell ( W_{n : 1} ( \cdot ) )$);
and
\emph{(ii)}~$( \sigma_r^2 ( \cdot ) )^{1 - 1 / n}$, by which the speed of~$\sigma_r ( \cdot )$ is proportional to its size exponentiated by~$2 - 2 / n$.
The second factor creates a momentum-like effect, which attenuates the movement of small singular values and accelerates the movement of large ones.
Accordingly, we may expect that with near-zero initialization (regime of interest), singular values progress slowly at first, and then, one after the other they reach a critical threshold and quickly rise, until convergence is attained.
Such dynamics can be viewed as a greedy learning process which incrementally increases the rank of its search space, thereby entailing a preference for solutions of low rank.
This greedy low rank learning process indeed takes place empirically, and moreover, in accordance with the fact that the exponent $2 - 2 / n$ grows with~$n$ (number of layers in the linear neural network), the process sharpens with depth---see Figure~\ref{fig:lnn_sing} for an example.
We note that under certain conditions, Theorem~\ref{theorem:sing_dyn} can be used to derive closed-form expressions for \smash{$( \sigma_r ( \cdot ) )_{r = 1}^{\min \{ d_0 , d_n \}}$} (see~\citet{arora2019implicit}), or to formally prove that the global minimizer to which $W_{n : 1} ( \cdot )$ converges has low rank (analogously to the results of~\citet{li2021towards,razin2021implicit,jin2023understanding}).

\begin{figure}[t]
\centering
\includegraphics[width=\textwidth]{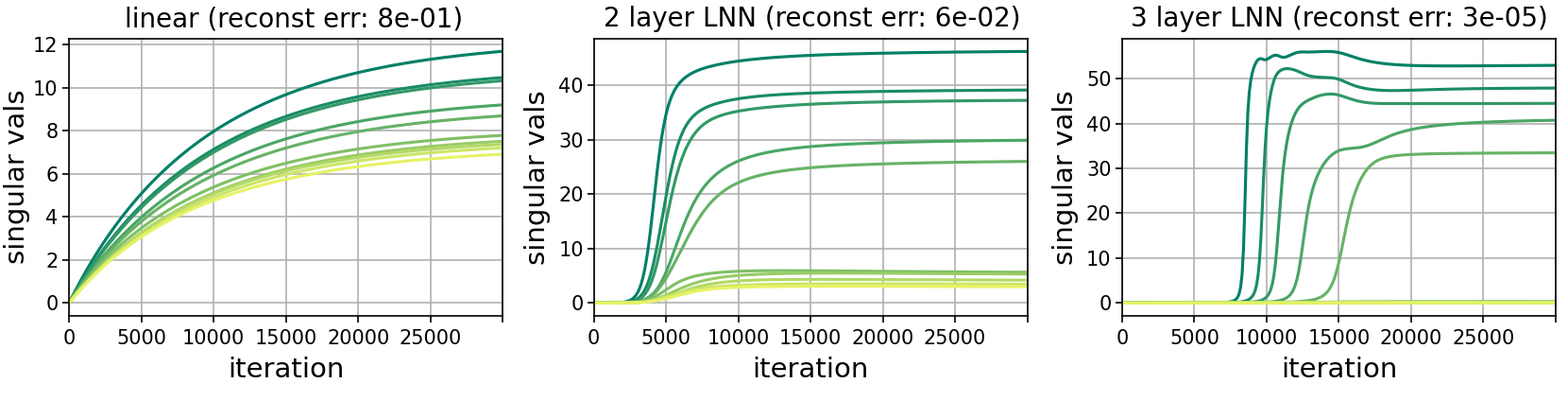}
\caption{
Empirical demonstration of the greedy low rank learning process brought forth by the implicit regularization of linear neural networks.
Plots show, for a matrix sensing task comprising measurements taken from a low rank ground truth, singular values of the learned solution throughout the iterations of gradient descent.
Left (``linear'') plot corresponds to direct minimization of the matrix sensing loss; middle (``2~layer LNN'') and right (``3~layer LNN'') plots correspond to minimization of the overparameterized objectives induced by two and three layer linear neural networks, respectively.
Plot titles specify reconstruction error (distance of learned solution from ground truth matrix) at the end of training.
Notice that the greedy low rank learning process takes place only with the linear neural networks, and is sharper with the deeper network.
For details see~\citet{arora2019implicit}, from which results are taken.
}
\label{fig:lnn_sing}
\end{figure}

\subsection{Implicit Compression by Overparameterization} \label{sec:gen:comp}

Recall from Subsection~\ref{sec:opt:acc} that overparameterization, \ie, replacement of a linear transformation with a linear neural network, can accelerate optimization.
The results of the current section imply that it also encourages convergence to solutions of low rank, meaning ones that can be compressed.\footnote{%
For $d , d' \in \N$, representing a linear transformation $\T : \R^d \to \R^{d'}$ generally requires $d d'$ parameters, but if $\rank ( \T ) = \rho$, where $\rho \ll \min \{ d , d' \}$, then far less parameters suffice, namely $\rho ( d + d' - \rho )$.
}
This phenomenon---an \emph{implicit compression by overparameterization}---facilitates a practical technique of adding linear layers for fitting training data with a reduced-size model, thereby improving both generalization and computational efficiency at inference time.
The technique has been applied to \emph{non-linear} neural networks in various real-world settings (see, \eg,~\citet{guo2020expandnets,jing2020implicit,huh2021low}).
Similarly to acceleration of optimization by overparameterization (see Subsection~\ref{sec:opt:acc}), it constitutes a practical application of linear neural networks born from a dynamical analysis.

\section{Extension: Arithmetic Neural Networks} \label{sec:ext}

Linear neural networks are a fundamental model in the theory of deep learning, and as discussed in Subsections \ref{sec:opt:acc} and~\ref{sec:gen:comp}, they also admit practical benefits.
Nevertheless, they are limited to linear (input-to-output) mappings, thus fail to capture the crucial role of non-linearity in deep learning.
In this section we briefly discuss a non-linear extension of linear neural networks that is closer to practical deep learning.

Linear neural networks can be viewed as \emph{matrix factorizations}, in accordance with the fact that the mappings they realize are naturally represented as matrices factorized by network weights.
Lifting matrices (two-dimensional arrays) to higher dimensions, \ie, considering \emph{tensor factorizations} (or more precisely, mappings represented as tensors factorized by network weights),\footnote{%
For a mathematical introduction to tensor analysis see~\citet{hackbusch2012tensor}.
}
gives rise to neural networks with multiplicative non-linearity, known as \emph{arithmetic neural networks}.
Arithmetic neural networks have been shown to exhibit promising empirical performance (see, \eg, \citet{cohen2016deep,sharir2016tensorial,chrysos2021deep}), and their expressiveness was the subject of numerous theoretical studies (see, \eg,~\citet{cohen2016expressive,cohen2016convolutional,cohen2017inductive,cohen2017analysis,cohen2018boosting,sharir2018expressive,levine2018benefits,levine2018deep,balda2018tensor,khrulkov2018expressive,khrulkov2019generalized,levine2019quantum,alexander2023makes,razin2023ability}).
It is possible to extend some of the (linear neural network) results in these lecture notes to arithmetic neural networks.
In particular, through a dynamical analysis, it can be shown that---similarly to how the implicit regularization of linear neural networks lowers matrix rank---the implicit regularization of arithmetic neural networks lowers \emph{tensor ranks}.
For details see \citet{razin2021implicit} and~\citet{razin2022implicit}.

\section{Conclusion}

These lecture notes presented a theory of linear neural networks---a fundamental model in the study of optimization and generalization in deep learning.
At the heart of the theory lies a dynamical characterization, by which training a linear neural network is equivalent to training a linear mapping with a certain preconditioner that promotes movement in directions already taken (Section~\ref{sec:dyn}).
The dynamical characterization facilitated two results concerning optimization (Section~\ref{sec:opt}):
\emph{(i)}~a guarantee of convergence to global minimum, applicable to linear neural networks of arbitrary depth (Subsection~\ref{sec:opt:cvg});
and
\emph{(ii)}~a proof that there exist cases where insertion of depth via linear layers can accelerate training, despite introducing non-convexity while yielding no gain in terms of expressiveness (Subsection~\ref{sec:opt:acc}).
With regards to generalization, the dynamical characterization was used to show that the implicit regularization of linear neural networks implements a greedy low rank learning process, which cannot be characterized as lowering any norm (Section~\ref{sec:gen}).
Practical applications born from the presented theory---namely, techniques for improving optimization, generalization and computational efficiency of \emph{non-linear} neural networks---were discussed (Subsections \ref{sec:opt:acc} and~\ref{sec:gen:comp}).

The dynamical approach underlying the presented theory was shown to overcome limitations of other theoretical approaches.
For example:
\emph{(i)}~it established a convergence guarantee (Theorem~\ref{theorem:lnn_converge}) in the presence of non-strict saddle points (Proposition~\ref{prop:oprm_obj_non_strict}), \ie, in a setting where generic landscape arguments from the literature on non-convex optimization (see, \eg,~\citet{ge2015escaping,lee2016gradient}) are invalid;
\emph{(ii)}~it proved that non-convex training of a linear neural network can be much faster than convex training of a linear mapping (Claim~\ref{claim:lnn_acc}), a conclusion that cannot be attained through the classical lens by which convex training is preferable;
and
\emph{(iii)}~it provided a description of implicit regularization (Theorem~\ref{theorem:sing_dyn}) applicable to settings where no norm is being lowered (Proposition~\ref{prop:imp_reg_no_norm}), \ie, where the standard association between regularization and norms (\cf~Supposition~\ref{supposition:imp_reg_nuc_norm}) is inappropriate.
Dynamical approaches have recently been adopted beyond the context of linear neural networks, \eg, for analyzing arithmetic neural networks (see Section~\ref{sec:ext}).
We hypothesize that such approaches will be key to developing a complete theoretical understanding of deep~learning.

\section*{Exercises}

\begin{exercise}
\label{exercise:e2e_dyn_single_out}
Consider the expression for the end-to-end dynamics given in Theorem~\ref{theorem:e2e_dyn} (Equation~\eqref{eq:e2e_dyn}).
Simplify this expression for the case of a single output variable (\ie,~$d_n = 1$), and explain how the simplified form resonates with the interpretation of the end-to-end dynamics as promoting movement in directions already~taken.
\end{exercise}

\vspace{1mm}
\begin{exercise}
\label{exercise:e2e_dyn_symm}
In this exercise you will derive a variant of the~end-to-end dynamics (Theorem~\ref{theorem:e2e_dyn}) that applies to a two layer \emph{symmetric} linear neural network, \ie, to the parametric family of hypotheses $\{ \x \mapsto W W^\top \x : W \in \R^{d \times d} \}$, where $d \in \N$.
Let $\ell_s : \R^{d \times d} \to \R$ be an analytic training loss, and consider the induced~objective:
\[
\phi_s : \R^{d \times d} \to \R
~ , ~
\phi_s ( W ) = \ell_s ( W W^\top )
\text{\,.}
\]
Suppose we run gradient flow over~$\phi_s ( \cdot )$, meaning we generate a continuous curve $W ( \cdot )$ in~$\R^{d \times d}$ via the following differential equation:
\[
\dot{W} ( t ) := \tfrac{d}{dt} W ( t ) = - \nabla \phi_s ( W ( t ) )
~ , ~
t > 0
\text{\,.}
\]
Derive a self-contained expression for the dynamics of the curve $W_s ( \cdot ) := W ( \cdot ) W ( \cdot )^\top$.
Specifically, derive an expression for $\dot{W}_s ( t ) := \tfrac{d}{dt} W_s ( t )$, where $t > 0$, that depends on~$W ( t )$ only via~$W_s ( t )$ (meaning the expression may include $W_s ( t )$ but not~$W ( t )$).
\end{exercise}

\vspace{1mm}
\begin{exercise}
\label{exercise:mat_dist}
Prove the following result, used in Example~\ref{example:def_margin}.
For any $d , d' \in \N$, $W \in \R^{d \times d'}$ and $\delta > 0$, it holds that:
\[
\min \big\{ \| W - W' \|_F : W' \in \R^{d \times d'} , \sigma_{\min} ( W' ) \leq \delta \big\} = \max \big\{ 0 , \sigma_{\min} ( W ) - \delta \big\}
\text{\,,}
\]
where $\| \cdot \|_F$ stands for Frobenius norm and $\sigma_{\min} ( \cdot )$ refers to the minimal singular value of a matrix.
You may use without proof the fact that $\| W - W' \|_F \geq \| \sigmabf ( W ) - \sigmabf ( W' ) \|_2$, where $\sigmabf ( W )$ and~$\sigmabf ( W' )$ are the vectors in~$\R^{\min \{ d , d' \}}$ holding, in non-increasing order, the singular values of $W$ and~$W'$, respectively (see Exercise~IV.3.5 in~\citet{bhatia1997matrix}).
\end{exercise}

\vspace{1mm}
\begin{exercise}
\label{exercise:imp_reg_no_norm_neg_det}
Suppose we modify the statement of \mbox{Proposition~\ref{prop:imp_reg_no_norm}} by replacing the condition $\det \big( W_{n : 1} ( 0 ) \big) > 0$ with $\det \big( W_{n : 1} ( 0 ) \big) < 0$.
Adapt the proof of the proposition such that it accords with the modified statement.
\end{exercise}

\vspace{1mm}
\begin{exercise}
Consider the singular value dynamics from Theorem~\ref{theorem:sing_dyn}, in the special case where the linear neural network has two layers, \ie, $n = 2$.
Let \smash{$( c_r )_{r = 1}^{\min \{ d_0 , d_2 \}}$} be a tuple of distinct real numbers.
Let $\epsilon > 0$ and $t_e > 0$.
Assume that, for $r = 1 , 2 , \ldots , \min \{ d_0 , d_2 \}$:
\begin{itemize}[topsep=0mm]
\item $\sigma_r ( \cdot )$~equals~$\epsilon$ at initialization (\ie, $\sigma_r ( 0 ) = \epsilon$);
and
\item $\langle - \nabla \ell ( W_{2 : 1} ( \cdot ) ) , \uu_r ( \cdot ) \vv_r^\top ( \cdot ) \rangle$ is fixed at~$c_r$ until time~$t_e$ of optimization (\ie, for all $t \in [ 0 , t_e ]$ it holds that $\langle - \nabla \ell ( W_{n : 1} ( t ) ) , \uu_r ( t ) \vv_r^\top ( t ) \rangle = c_r$).
\end{itemize}
Derive closed-form expressions for $( \sigma_r ( t ) )_{r = 1}^{\min \{ d_0 , d_2 \}}$, where $t \in [ 0 , t_e ]$.
Show that gaps between singular values grow rapidly, in the sense that, for any $r , r' \in \big\{ 1 , 2 , \ldots , \min \{ d_0 , d_2 \} \big\}$, $r \neq r'$, the ratio between $\sigma_r ( t )$ and~$\sigma_{r'} ( t )$ evolves exponentially with~$t$ when $t \in [ 0 , t_e ]$.
\end{exercise}

%% file: ack.tex
We thank Yotam Alexander, Nimrod De La Vega, Itamar Menuhin-Gruman, and Tom Verbin for aid in writing some of the proofs.
The theory presented in these notes was developed with the support of NSF, ONR, Simons Foundation, Schmidt Foundation, Mozilla Research, Amazon Research, DARPA, SRC, Len Blavatnik and the Blavatnik Family Foundation, Yandex Initiative in Machine Learning, a Google Research Scholar Award, a Google Research Gift, Israel Science Foundation, Tel Aviv University Center for AI and Data Science, and Amnon and Anat Shashua.
NR is supported by the Apple Scholars in AI/ML PhD fellowship.